\documentclass[final,lettersize,conference,10pt]{IEEEtran}
\IEEEoverridecommandlockouts
\usepackage[compact]{titlesec}
\titlespacing{\section}{0pt}{2ex}{1ex}
\titlespacing{\subsection}{0pt}{1ex}{0ex}
\usepackage{cite}
\usepackage{flushend}
\usepackage{empheq}
\usepackage{amsmath, amssymb, amsfonts, cases}
\usepackage{dsfont}
\usepackage{hyperref}
\usepackage{amsthm}
\usepackage{algorithm,algorithmic}
\usepackage{graphicx}
\usepackage{textcomp}
\usepackage[caption=false]{subfig}
\usepackage{setspace}
\usepackage{float}
\usepackage{xcolor}
\usepackage{mathtools}
\usepackage{cite}
\usepackage{times}
\usepackage{color}
\usepackage{multirow, siunitx}
\usepackage{multicol}
\usepackage{booktabs}
\usepackage{stackrel, url}
\usepackage{lipsum, enumitem}
\theoremstyle{remark}
\newcommand\scalemath[2]{\scalebox{#1}{\mbox{\ensuremath{\displaystyle #2}}}} 
\newtheorem{definition}{\bf \emph{Definition}}
\newtheorem{theorem}{\bf \emph{Theorem}}
\newtheorem{lemma}{\bf \emph{Lemma}}

\newtheorem{remark}{\textit{Remark}}

\newtheorem{assumption}{\bf\emph{Assumption}}
\def\cA{{\mathcal{A}}} \def\cB{{\mathcal{B}}} \def\cC{{\mathcal{C}}} \def\cD{{\mathcal{D}}}
 \def\cF{{\mathcal{F}}} \def\cG{{\mathcal{G}}} 
   
 \def\cN{{\mathcal{N}}}  
\def\cQ{{\mathcal{Q}}} \def\cR{{\mathcal{R}}} \def\cS{{\mathcal{S}}} 
   \def\cX{{\mathcal{X}}}

\def\b({ \bigg( }
\def\b){ \bigg) }

\def\b[{\bigg[}
\def\b]{\bigg]}

\def\ba{{\mathbf{a}}} \def\bb{{\mathbf{b}}}   \def\be{{\mathbf{e}}} 
 \def\bg{{\mathbf{g}}}   
  \def\bm{{\mathbf{m}}} \def\bn{{\mathbf{n}}} 
    
 \def\bv{{\mathbf{v}}} \def\bw{{\mathbf{w}}} \def\bx{{\mathbf{x}}} \def\by{{\mathbf{y}}}
\def\bz{{\mathbf{z}}} 

  \def\bC{{\mathbf{C}}}  
   \def\bI{{\mathbf{I}}}

  \def\R{{\mathbb{R}}}        \def\E{\mathbb{E}}

\def\argmin{\mathop{\mathrm{argmin}}}

     \def\d4{\!\!\!\!}            \def\ep{\epsilon}

\def\lL{\langle}  \def\lR{\rangle}
  \def\-{\! - \!}  \def\+{\! + \!}  \def\={\! = \!}  \def\>{\! > \!}

\newcommand{\bef}{\begin{figure}}
\newcommand{\eef}{\end{figure}}
\newcommand{\beq}{\begin{eqnarray}}
\newcommand{\eeq}{\end{eqnarray}}

\def\sumM{\sum_{m=1}^{M}}

\def\Wnabla{\widetilde{\nabla}}

\def\wnabla{\widetilde{\nabla}}

\def\whw{\check{\bw}}

\def\bbw{\overline{\bw}}
\def\wst{ \bw^{\star}}

\def\indi{\mathds{1}}

\begin{document}
\title{Privacy-Preserving Quantized Federated Learning with Diverse Precision}

\author{
\IEEEauthorblockN{
Dang~Qua~Nguyen,~\IEEEmembership{Student Member,~IEEE}, Morteza Hashemi,~\IEEEmembership{Senior Member,~IEEE}, \\ Erik Perrins,~\IEEEmembership{Senior Member,~IEEE},  Sergiy A. Vorobyov,~\IEEEmembership{Fellow,~IEEE}, David J. Love,~\IEEEmembership{Fellow,~IEEE}, \\ and Taejoon~Kim,~\IEEEmembership{Senior Member,~IEEE}
}

\thanks{
D. Q. Nguyen, M. Hashemi, and E. Perrins are with the Department of Electrical Engineering and Computer Science, The University of Kansas, Lawrence, KS 66045 USA (e-mail: \{quand, mhashemi, esp\}@ku.edu). Sergiy A. Vorobyov is with the Department of Information and Communications Engineering, Aalto University, 02150 Espoo, Finland (e-mail:
sergiy.vorobyov@aalto.fi). D. J. Love is with the School of Electrical and Computer
Engineering, Purdue University, West Lafayette, IN 47907 USA (e-mail:djlove@purdue.edu). Taejoon Kim is with the School of Electrical, Computer and Energy Engineering, Arizona State University, Tempe, AZ  85287 USA (e-mail: taejoonkim@asu.edu).

This work was supported in part by NSF under grants CNS2451268, CNS2225578, and CNS2514415; ONR under grant N000142112472; the NSF and Office of the Under Secretary of Defense (OUSD) – Research and Engineering, Grant ITE2515378, as part of the
NSF Convergence Accelerator Track G: Securely Operating Through 5G Infrastructure Program; and in part by the Research Council of Finland under Grant 357715.
}
}
\maketitle	
%\vspace{-5mm}
\begin{abstract}
Federated learning (FL) has emerged as a promising paradigm for distributed machine learning, {enabling collaborative training of a global model across multiple local devices without requiring them to share raw data.} 
Despite its advancements, FL is limited by factors such as: (i) privacy risks arising from the unprotected transmission of local model updates to the fusion center (FC) and (ii) decreased learning utility caused by heterogeneity in model quantization resolution across participating devices. 
Prior work typically addresses only one of these challenges because maintaining learning utility under both privacy risks and quantization heterogeneity is a non-trivial task. 
In this paper, our aim is therefore to improve the learning utility of a privacy-preserving FL that allows clusters of devices with different quantization resolutions to participate in each FL round. 
Specifically, we introduce a novel stochastic quantizer (SQ) that is designed to simultaneously achieve differential privacy (DP) and minimum quantization error. 
Notably, the proposed SQ guarantees bounded distortion, unlike other DP approaches. 
To address quantization heterogeneity, we introduce a cluster size optimization technique combined with  a linear fusion approach to enhance model aggregation accuracy.
Numerical simulations validate the benefits of our approach in terms of privacy protection and learning utility compared to the conventional LaplaceSQ-FL algorithm.
\end{abstract}
\section{Introduction}
Federated learning (FL) is an efficient machine learning (ML) framework, allowing local devices to collaboratively train a global model \cite{VSmith2020, YeLi2021, VPoor2022}. 
Unlike conventional centralized ML, which requires transferring training data to a fusion center (FC), FL retains training data on distributed devices and trains models locally.  
Then, the trained local model updates are sent to the FC for global aggregation. 
FL has been widely adopted in applications such as next-word prediction \cite{stremmel2021, lin2021fednlp}, autonomous vehicles \cite{Chris2023, Dinh2021}, eHealthcare \cite{ zhang2023, Vladimir2023}, and large language models (LLMs) \cite{mcmahan2017LLM, ye2024llm}. 
Despite its widespread applicability, FL encounters several challenges.
First, FL is susceptible to privacy leakage \cite{Kairouz2021}. 
As demonstrated by model inversion attacks \cite{Zhu2019, huang2021, Jonas2020}, an \mbox{adversary} eavesdropping on local model updates can \mbox{reconstruct} sensitive local training data. 
These attacks are difficult to detect due to their passive and nonintrusive nature.
Second, the inherent bandwidth and power limitations of FL networks impose communication constraints, necessitating the quantization of local model updates by devices.
In practice, the heterogeneity of these devices often exhibits varying quantization resolutions, leading to inconsistent model update quality that negatively impacts learning utility \cite{Kairouz2021, VPoor2022}.
Therefore, simultaneously addressing data privacy and quantization heterogeneity is crucial for the effective deployment of FL.
\subsection{Related Works}
While various approaches have been proposed to mitigate privacy leakage or to incorporate quantization resolution \mbox{heterogeneity}, they often focus on one issue at the expense of the other\cite{Kairouz2021, VPoor2022}. 
For instance, differential privacy (DP) \cite{Dwork2014} has been adopted in FL \cite{Kang2020, Koda2020, Simeone2021, Tandon2021, Muah2021, Wei2022, Qua2023Asilomar}  to reduce privacy leakage by adding artificial noise to model updates. 
There is a fundamental tradeoff between privacy and learning utility: increasing noise power improves privacy at the cost of utility, and vice versa.\looseness=-1
Other privacy-preserving approaches apply stochastic quantization (SQ) methods to local model updates \cite{Chen2021, youn2023randomized, Wang2024ICC, Lang2023}. 
Unlike conventional DP mechanisms \cite{Kang2020, Koda2020, Simeone2021, Tandon2021, Muah2021, Wei2022, Qua2023Asilomar}, these methods \cite{Chen2021, youn2023randomized, Wang2024ICC, Lang2023} exploit the inherent randomness of a stochastic quantizer to offer a DP guarantee.
However, previous works \cite{Chen2021, youn2023randomized, Wang2024ICC, Lang2023} focus primarily on achieving a DP guarantee without mitigating the quantization error.
For example, the SQ in \cite{Lang2023} could mimic the Laplace DP mechanism \cite{Dwork2014}, but it results in a degradation of the utility under strict DP requirements.
In general, these previous works \cite{Kang2020, Koda2020, Simeone2021, Tandon2021, Muah2021, Wei2022, Qua2023Asilomar, Chen2021, youn2023randomized, Wang2024ICC, Lang2023} neglect the heterogeneity in quantization resolution between local devices.
Recent studies \cite{CongShen2021, CongShen2024} have addressed quantization \mbox{heterogeneity} by proposing adaptive fusion methods that assign larger weights to model updates from higher resolution devices. 
%\mbox{minimize} learning errors by assigning larger weights to model updates from higher resolution devices.
%%
While these approaches \cite{CongShen2021, CongShen2024} are plausible, the underlying assumption {is that link conditions are ideal.}
In practice, unreliable links introduce noise into quantized model updates and assigning larger weights to these noisy updates can degrade the learning utility.
Effectively balancing learning utility with privacy under quantization heterogeneity in  FL systems is a non-trivial task requiring a unified framework that leverages techniques from data privacy, communication theory, and signal processing.
\subsection{Summary of Contributions} 
The goal of our work is to improve the learning utility of privacy-preserving FL in the presence of heterogeneous quantization resolution across devices.
To achieve this, we develop a unified framework that reduces quantization error and fusion error, consequently minimizing learning error while ensuring DP.
In our FL network, devices are partitioned into groups, wherein devices within the same group share a common quantization resolution, as detailed later in Section~\ref{SecIVA}.
Due to the limited bandwidth of the communication links, only a cluster of devices from each group is allowed to participate in each FL round \cite{VPoor2022, Chang2022}. \looseness=-1

The main contributions of this paper are summarized below. 
\begin{itemize}[leftmargin = 6mm, itemsep = 0.01in, parsep = 0.03in, topsep = 0.01in, partopsep = 0.05in]
\item We design a novel SQ that minimizes quantization distortion while preserving DP guarantees. 
The proposed DP-preserving SQ maps the source to the nearest quantization level with a certain probability, in which the probability is determined by the DP requirement. 
Unlike prior quantizers \cite{Chen2021, youn2023randomized, Wang2024ICC, Lang2023}, our approach explicitly minimizes quantization distortion while ensuring a DP guarantee. 
Our analysis indicates that the distortion associated with the proposed quantizer is bounded, in contrast to the prior SQ method \cite{Alistarh2017} combined with Laplace mechanism \cite{Dwork2014} that leads to unbounded distortion.\looseness=-1
\item To maintain learning utility under quantization heterogeneity, we propose an approach that integrates cluster sizes optimization with a linear fusion technique. 
The cluster sizes optimization improves learning utility by minimizing the learning error upper bound, obtained through the convergence analysis of the proposed DP-preserving FL, while the linear fusion method enhances model aggregation accuracy across devices with varying quantization resolutions. 
Our convergence analysis reveals that the learning error bound scales linearly with cluster sizes. 
We formulate the cluster size optimization as a linear integer programming (LIP) problem and identify optimal cluster sizes.  
\item Numerical simulations are presented to demonstrate the advantages of the proposed approach in terms of data privacy protection and learning utility. 
These results confirm the superior capability of the proposed approach in protecting data privacy while maintaining learning utility, compared to existing approaches.
\end{itemize}
The remainder of this paper is organized as follows.
Section~\ref{SecII.B} proposes a novel SQ for ensuring a DP guarantee.
{Section~\ref{SecIV} presents the network model, the proposed FL algorithm, and the fusion technique.}
%%% 
Section~\ref{SecV} analyzes the convergence behavior of the proposed approach and optimizes the cluster sizes.
Section~\ref{SecVI} demonstrates the benefits of the proposed framework through numerical simulations.
The paper is concluded in Section~\ref{SecVII}.
\textit{Notation}: $\ba = [a_1,\dots,a_d]^T \in \R^{d \times 1}$ denotes a column vector with the $i$th entry $a_i$, for $i=1,\dots,d$. 
$\|\ba\|_p$ denotes the $\ell_p$-norm of $\ba$. 
$|\cA|$ denotes the cardinality of the set $\cA$.
$\bI$ represents the identity matrix with appropriate dimensions. 
$\langle \bx, \by \rangle$ denotes the inner product of $\bx$ and $\by$. 
$\E[\cdot]$ represents the expectation. 
$\mathbf{0}$ denotes the all-zeros vector with an appropriate dimension. 
$\cN(\bm,\bC)$ denotes the Gaussian distribution with mean $\bm$ and covariance matrix $\bC$.
$\nabla f$ and $\Wnabla f$ denote the gradient and stochastic gradient of $f$, respectively.
The indicator function $\indi_\cX(x) = 1$ if $x \in \cX$ and $\indi_{\cX}(x) = 0$, otherwise. \looseness=-1
%%
% % Table summarizing notations and definitions
% \begin{table}[h]
% \caption{Notations and Definitions for FL System}
% \centering
% \label{table1}
% \begin{tabular}{|c|p{6cm}|}
% \hline
% \textbf{Notation} & \textbf{Definition} \\
% \hline
% $K$ & Total number of devices \\
% \hline
% $N$ & Total number of participating devices per FL round \\
% \hline
% $\cG_m$ & Group of devices using $b_m$-bit quantization \\
% \hline
% $\cD_{m,u}$ & Data set of the $u$th device in $\cG_m$\\
% \hline
% $f_{m,u}$ & Loss function of the $u$th device in $\cG_m$ \\
% \hline
% $f$ & Global loss function \\
% \hline
% $\nabla f_{m,u}(\bx)$ & Gradient of $f_{m,u}(\bx)$ \\
% \hline
% $\Wnabla f_{m,u}(\bx)$ & Stochastic gradient of $f_{m,u}(\bx)$ \\
% \hline
% $T$ & Number of FL rounds\\
% \hline
% $L$ & Number of local stochastic gradient descent (SGD) iterations\\
% \hline
% $\bw^t$ & Global model at the $t$th FL round\\
% \hline
% \end{tabular}
% \end{table}
% %%
\section{Differentially~Private~Stochastic~{Quantization}}\label{SecII.B}
This section first presents the definition of DP. 
Then, we propose our differentially private SQ approach and analyze its minimum distortion.
\subsection{Differential Privacy}
We begin with introducing necessary definitions.  
\begin{definition}\label{neighborDF}
    Two finite data sets $\cA$ and $\cA'$ are neighbors if they differ by only one data point.
\end{definition}
\begin{definition}\label{DPdef}($\epsilon$-DP \cite{Dwork2014}) A randomized mechanism ${\mathsf{M}}: \mathbb{D} \rightarrow \cR$, where $\mathbb{D}$ is the domain of data sets and $\cR$ is the range of $\mathsf{M}$, satisfies $\epsilon$-DP if for any two neighbor data sets $\cA, \cA' \in \mathbb{D}$ and for any subset of output $\cS \subseteq \cR$, the following holds
\begin{equation}
\label{eq:DPdef}
    \Pr[\mathsf{M}(\cA) \in \cS] \leq e^{\epsilon} \Pr[\mathsf{M}(\cA') \in \cS], 
%    \vspace{-2mm}
\end{equation}where $\epsilon\geq 0$ denotes the privacy loss. 
\end{definition}
\noindent The $\epsilon$-DP criterion in \eqref{eq:DPdef} indicates that smaller $\epsilon$ provides stronger privacy protection, making it harder to infer the presence of any data point in $\cA$, even if the adversary has knowledge of the neighboring data set $\cA'$.
We present below two useful lemmas.
\begin{lemma}(Sequential composition \cite[Theorem 3.16]{Dwork2014})\label{lm:composition1}
    Let $\mathsf{M}_i$ be an $\epsilon_i$-DP mechanism for $i=1,\dots,d$. 
    Then, $\mathsf{M} = (\mathsf{M}_1(\cA),\dots,\mathsf{M}_d(\cA))$ is an $(\sum_{i=1}^{d} \epsilon_i)$-DP mechanism.
\end{lemma}
\noindent In Section~\ref{SecIV}, we apply Lemma~\ref{lm:composition1} to a composite DP mechanism formed by combining multiple DP mechanisms.
\begin{lemma}(Laplace mechanism \cite[Theorem 3.6]{Dwork2014})\label{LpMechanism}
    Given any function $h:\mathbb{D} \rightarrow \R^{d \times 1}$, the Laplace mechanism
    % \vspace{-2mm}
    % \begin{equation*}\small
    %     \mathsf{L}(h(\cA)) = h(\cA) + \bz,
    %     \vspace{-2mm}
    % \end{equation*} 
    $\mathsf{L}(h(\cA)) = h(\cA) + \bz$,
    ensures $\epsilon$-DP, where $\bz = [z_1,\dots,z_d]^T \in \R^{d \times 1}$ and $\{z_i\}_{i=1}^{d}$ are independent and identically distributed (i.i.d.) Laplace random variables with mean $0$ and variance $\frac{2\rho^2}{\epsilon^2}$, where $\rho$ is the $\ell_1$-sensitivity, defined for {any two neighbor data sets $\cA, \cA' \in \mathbb{D}$ as}  
%    \vspace{-2mm}
    \begin{equation}\label{eq:l1sen}
        \rho = \max_{\cA, \cA'}\|h(\cA) - h(\cA')\|_1.
%    \vspace{-2mm}
    \end{equation}
\end{lemma}
{
It is noteworthy that the Laplace mechanism in Lemma 2 can be combined with a SQ method. 
Such a mechanism, which we call LaplaceSQ $\mathsf{L(a)}$, is formed by combining  
the prior $b$-bits stochastic quantizer $Q_b(a)$ in \cite{Alistarh2017} with the standard Laplace mechanism in Lemma~\ref{LpMechanism} with $h(a) = Q_b(a)$, in which $a \in \R$  is a uniform random variable drawn from the interval $[\underline{a}, \overline{a}]$ with $\underline{a} \leq \overline{a}$.
In particular, 
{\begin{equation}\label{eqLaplaceSQ}
    \mathsf{L}(a) = Q_b(a) + z,
\end{equation}}where $z$ is a Laplace noise with mean $0$ and variance $\frac{2\rho^2}{\epsilon^2_1}$ to achieve $\epsilon_1$-DP guarantee.
Taking the LaplaceSQ as a benchmark, we compare it with our SQ method proposed in the next subsection.} 
\subsection{Differentially Private Stochastic Quantization}\label{secIIIB}
We first define the set of {real-valued quantization levels} $\{q_1, q_2, \dots, q_{2^b}\}$, uniformly spaced in the interval $[\underline{a}, \overline{a}]$, where $q_1 <q_2 <\dots <q_{2^b} $.
Each level $q_j$ is then written as $q_j = \underline{a} + (j-1)\frac{\overline{a} - \underline{a}}{2^b-1}$, for $j=1,\dots,2^b$.
Given that $a$ lies within the interval $[q_i, q_{i+1})$ for some $i$, a SQ mechanism assigns $a$ to one of the two quantization levels: $q_i$ or $q_{i+1}$ according to a probabilistic mapping.
Specifically, the stochastic quantizer $\cQ_b(a)$ is given by \looseness=-1
    \vspace{-1mm}
    \begin{subnumcases}{\label{eq2} \cQ_b(a) = }
     q_i,& \text{ with probability }     {$p$}, \label{eq2a}\\ 
      q_{i+1},& \text{ with probability } {$1-p$}, \label{eq2b}
    \vspace{-1mm}
    \end{subnumcases}
where $0 \leq  p \leq 1$.
If the input is a vector $\ba \in \R^{d\times 1}$, the quantizer in \eqref{eq2} is applied element-wise, i.e., $\cQ_b(\ba) = [\cQ_b(a_1),\dots,\cQ_b(a_d)]^T\in \R^{d \times 1}$. 
It is noted that prior stochastic quantizers \cite{Alistarh2017, Hassani2019, Chang2021} designed $p$ in \eqref{eq2} without providing DP guarantees \cite{Wang2024ICC}.
Unlike \cite{Alistarh2017, Hassani2019, Chang2021}, we establish an optimization framework for determining $p$ in \eqref{eq2} to ensure an $\epsilon_1$-DP guarantee while minimizing the expected quantization distortion $\E[|\cQ_{b}(a) - a|^2]$, {leading to privacy-preserving quantizer $\cQ_{b,\epsilon_1}(a)$}. 
The optimization problem is therefore given by  \looseness=-1
\vspace{-2mm}
\begin{subequations}
\label{eq3}
    \beq
    \min_{p} && p(q_i-a)^2 + (1-p)(q_{i+1}-a)^2, \label{eq3a}\\ 
    \text{subject to } && e^{-\epsilon_1} \leq \frac{p}{1-p} \leq e^{\epsilon_1},  \label{eq3c}
    \vspace{-2mm}
    \eeq 
\end{subequations}where the objective function \eqref{eq3a} is the expected quantization distortion. 
The constraint  \eqref{eq3c} ensures the $\epsilon_1$-DP requirement because $\frac{\Pr[\cQ_{b,\epsilon_1}(a) = q_i]}{\Pr[\cQ_{b,\epsilon_1}(a')= q_i] } \leq e^{\epsilon_1}$ and $\frac{\Pr[\cQ_{b,\epsilon_1}(a) = q_{i+1}]}{\Pr[\cQ_{b,\epsilon_1}(a')= q_{i+1}]} \leq e^{\epsilon_1}$, for $a,a' \in [q_i, q_{i+1})$.
\looseness=-1
The problem in \eqref{eq3} is solved using bounded linear programming (BLP), which admits an optimal solution that occurs at the boundary of its feasible region $\cF = \{x \mid 0 \leq x \leq 1,~e^{-\epsilon_1} \leq \frac{x}{1-x} \leq e^{\epsilon_1}\}$ \cite{wagner1958dual}. 
The simplex method \cite{luenberger1984} or interior-point method \cite{karmarkar1984} are commonly employed to obtain optimal solution, in which the simplex method exhibits combinatorial complexity, while the interior-point method exhibits polynomial complexity.
Instead of solving \eqref{eq3} numerically, the following lemma identifies the closed-form solution for the BLP in \eqref{eq3}.\looseness=-1
\begin{lemma}
    \label{lm1}
    The optimal solution to problem \eqref{eq3} for the stochastic quantizer $\cQ_{b,\epsilon_1}(a)$ is given by
  %      \vspace{-2mm}
    \begin{subnumcases}{\label{eq4}p^{\star} = }\!\!
      \scalemath{1}{\frac{e^{\epsilon_1}}{e^{\epsilon_1}+1},} \text{ if $|q_i-a| \leq  |q_{i+1}-a|$}, \label{4a}\\ 
    \!\!\scalemath{1}{\frac{1}{e^{\epsilon_1}+1}}, \text{ otherwise}. \label{eq4b}
    \end{subnumcases}
\end{lemma}
\begin{proof}
 %  \vspace{-2mm}
   See Appendix~\ref{ProofLM1}.
  % \vspace{-0.5cm}
\end{proof}
Lemma~\ref{lm1} indicates that the optimal probabilities are determined by the $\epsilon_1$-DP constraint. 
Specifically, when $\epsilon_1 = 0$, representing the strongest possible DP guarantee, the proposed quantizer $\cQ_{b,0}(a)$ maps $a$ to either $q_i$ or $q_{i+1}$ with equal probability of $0.5$.
Conversely, as $\epsilon_1 \rightarrow \infty$, representing the weakest DP guarantee, the proposed quantizer $\cQ_{b,\infty}(a)$ maps $a$ to the nearest quantization level with probability $1$. 
\subsection{Quantization Distortion Analysis} \label{SecII.C}
Lemma~\ref{lm1} is instrumental in computing the minimum distortion achieved by the optimization problem \eqref{eq3}.  
Using the closed-form expression \eqref{eq4} and after some algebraic manipulation, the expected distortion in \eqref{eq3a} can be expressed as 
\vspace{-2mm}
\begin{multline}
    \label{eq:g_eps1}
    \d4\scalemath{0.9}{\E[|\cQ_{b,\ep_1}(a) - a|^2] =
    \frac{e^{\epsilon_1}\min\{(q_{i}-a)^2,(q_{i+1}-a)^2\}}{e^{\epsilon_1} + 1}} \\ ~~~\scalemath{0.9}{+ \frac{\max\{(q_{i}-a)^2,(q_{i+1}-a)^2\}}{e^{\epsilon_1} + 1}}.
\vspace{-2mm}
\end{multline}
\begin{remark}\label{rmk1} 
Compared to the LaplaceSQ mechanism, the proposed quantizer $\cQ_{b,\epsilon_1}(a)$ in Lemma~\ref{lm1} ensures an $\epsilon_1$-DP guarantee while achieving the minimal quantization distortion.  
To see this, we first define $\E[|\mathsf{L}(a) -a|^2]$ as the expected distortion of LaplaceSQ. 
It is not difficult to observe that as $\epsilon_1 \rightarrow 0$, the distortion $\E[|\mathsf{L}(a) -a|^2] = \E[|Q_b(a)-a|^2] + \frac{2\rho^2}{\epsilon^2_1}$ grows unbounded. 
On the other hand, given the same $\epsilon_1$-DP guarantee as the LaplaceSQ mechanism, the proposed quantizer $\cQ_{b,\epsilon_1}(a)$ in Lemma~\ref{lm1} maintains a bounded distortion as $\epsilon_1 \to 0$.
This property stems from the fact that the distortion in \eqref{eq:g_eps1} is a monotonically decreasing function of $\epsilon_1$, converging to its maximum value $\frac{(q_{i}-a)^2 + (q_{i+1}-a)^2}{2}$ as $\epsilon_1 \rightarrow 0$. \looseness=-1
\end{remark}
\section{Network Model and Privacy-Preserving Quantized FL Algorithm}\label{SecIV}
This section first presents our FL network model, which consists of an FC and multiple groups of distributed devices with heterogeneous quantization resolutions, and proposes the privacy-preserving quantized FL algorithm.
{Then, we analyze the DP guarantee of the proposed algorithm and design fusion weights for improved model aggregation.}
%%
%\vspace{-3mm}
\begin{figure}[htb]
    \centering
    \includegraphics[width=0.45\textwidth]{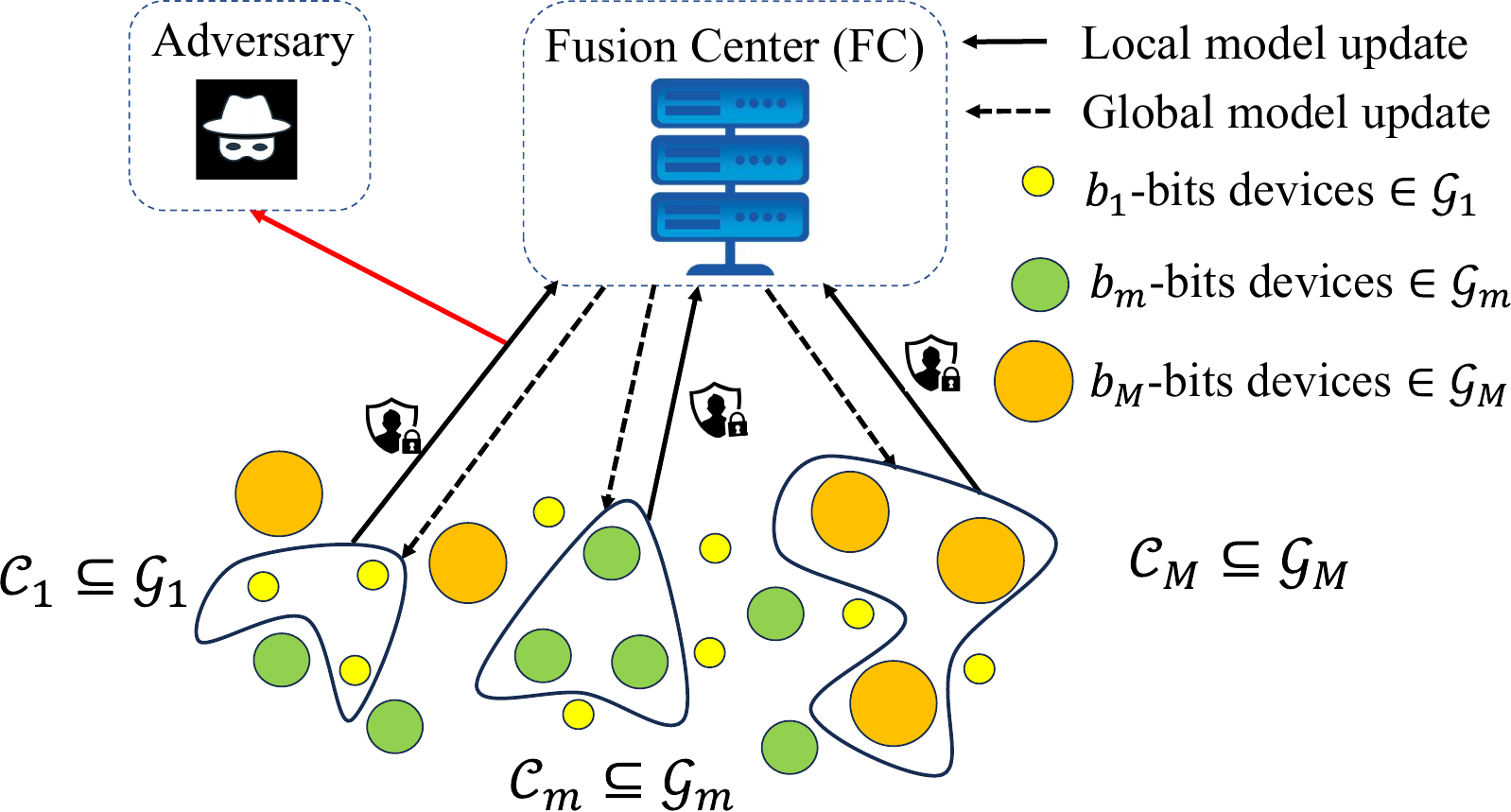}
     \vspace{-2mm}
    \caption{Privacy-preserving quantized FL network with quantization resolution heterogeneity across devices.}
    \label{fig:1}
 %   \vspace{-4mm}
\end{figure}
\subsection{Network Model}\label{SecIVA}
As shown in Fig.~\ref{fig:1}, our FL network incorporates heterogeneity in quantization resolution and consists of an FC and $M$ groups of distributed devices, $\cG_1, \dots, \cG_M$, where each group $\cG_m$ contains $|\cG_m| =g_m$ devices. 
Each device in $\cG_m$ uses $b_m$ bits for quantization with $1\leq b_1 \leq b_2 \leq \dots \leq b_M$. 
Specifically, $\cG_m = \{d_{m,1}, \dots, d_{m, g_m}\}$, where $d_{m,u}$ is the notation used to indicate the $u$th device in $\cG_m$.
The FC and the distributed devices collaboratively train a global ML model $\bw^{\star} \in \R^{d\times 1}$ through multiple FL rounds \cite{mcmahan2017} such that 
\vspace{-2mm}
\begin{equation}
    \label{eq1}
    \bw^{\star} = \argmin_{\bw} f(\bw), 
\vspace{-2mm}
\end{equation}where {\small $f(\bw) = \frac{1}{K}\sumM \sum_{u=1}^{g_m} f_{m,u}(\bw)$} is the global loss function, {\small $K = \sumM g_m$} is the total number of devices, and $f_{m,u}(\cdot)$ is the local loss function of device $d_{m,u}$. %for $u = 1,\dots,g_m$ and $m=1,\dots,M$.
Device $d_{m,u}$ utilizes its data set $\cD_{m,u}$ to locally train the ML model and then sends its model update to the FC.
We assume that elements in $\{\cD_{m,u}\}$ are i.i.d. across devices with $|\cD_{m,u}| = D$.  
The local loss function $f_{m,u}(\bw)$ is given by $f_{m,u}(\bw) = \frac{1}{D}\sum_{x \in \cD_{m,u}} \mathcal{\ell}(\bw,x)$, where $\mathcal{\ell}(\cdot)$ is the training loss function. 
%\tk{I changed $\mathsf{L}$ to $\ell$ because the former has been reserved for the Laplace mechanism. I had this comment but has not been addressed.}
%%

%%
 Due to the limited bandwidth of the device-to-FC links, only a cluster of devices $\cC_m \subseteq \cG_m$ within a group is allowed to join each FL round, where $|\cC_m| = c_m \leq g_m$ and $\sumM c_m = N \leq K$. 
This means the sum data rate of the model transfer from the devices to the FC is limited by $B$ bits, 
\vspace{-2mm}
\begin{equation}\small
    \label{eq:defB}
    \sumM c_m b_m \leq B.
\vspace{-2mm}
\end{equation} 
Devices in cluster $\cC_m$ are obtained by uniformly sampling $\cG_m$ with a probability of $\frac{1}{g_m}$.
We consider a threat model in which an adversary can eavesdrop on model updates from devices, as illustrated in Fig.~\ref{fig:1}, and deploy model inversion attacks \cite{Jonas2020, Zhu2019, huang2021} to reconstruct the local training data without altering the FL operations.
These attacks are particularly difficult to detect due to their passive and nonintrusive nature.
%%

%%
%\vspace{-3mm}
\begin{algorithm}[htp]
   \caption{Privacy-Preserving Quantized FL with Heterogeneous Quantization Resolution}  
   \label{algorithm1}
\begin{algorithmic}[1]
 \renewcommand{\algorithmicensure}{\textbf{Output:}}
 \REQUIRE $\{\cG_m\}_{m=1}^{M}$, $\{\eta_t\}_{t=0}^{T-1}$, $T$, $L$, $C$, $\{\epsilon_{1,m,u}\}$  
 \ENSURE $\bw^{T}$  
 \STATE \textbf{Initialization:} $\bw^{0}$ and $t = 0$
    \WHILE{$t < T$}   
    \STATE{FC randomly selects $\cC_m \subseteq \cG_m$, for $m=1,\dots,M$}\label{step3}
    \STATE{FC sends $\bw^t$ to devices in $\cC_m$, for $m=1,\dots,M$}\label{step4}
    \FOR{$d_{m,u} \in \cC_m$}
    \STATE{Local model update: setting $\bw^{0,m,u,t} = \bw^t$}\label{step6}
    \FOR{$l = 0,1,\dots,L-1$} 
    \STATE{$\bw^{l+1,m,u,t} = \bw^{l,m,u,t} - \eta_t\widetilde{\nabla}f_{m,u}(\bw^{l,m,u,t})$}\label{step8}
    \ENDFOR \label{step9}  
    \STATE{Model difference:}{ $\bv^{m,u,t} = \bw^{L,m,u,t}-\bw^{t}$}\label{step10}
    \STATE{Clipping: {\small $\widehat{\bv}^{m,u,t} = \min\left\{1,\frac{C}{\| \bv^{m,u,t}\|_1} \right\} \bv^{m,u,t}$}} \label{step11}
    \STATE{Stochastic quantization:} 
    {{\small $\widehat{\bw}^{m,u,t} = \cQ_{b_m,\epsilon_{1,m,u}}(\widehat{\bv}^{m,u,t})$}, where $\cQ_{b_m,\epsilon_{1,m,u}}$ is proposed in \eqref{eq4}} \label{step12}  
    \STATE{Send $\widehat{\bw}^{m,u,t}$ to the FC}  \label{step13}
    \ENDFOR 
    \STATE{Model fusion: }{ {\small $\bw^{t+1} = \bw^t + \sumM \sum_{d_{m,u} \in \cC_m} \omega_{m,u}(\widehat{\bw}^{m, u, t} + \bn^{m,u,t})$}}, where $\{\omega_{m,u}\}$ are fusion weights (see \eqref{eq:fusionWOpt} in Section~\ref{secFWOpt}).\label{step15}
    \STATE{$t = t+1$}
    \ENDWHILE
\end{algorithmic}
\end{algorithm}
\subsection{Privacy-Preserving Quantized FL Algorithm with Heterogeneous Quantization Resolution}
In what follows, we describe our proposed privacy-preserving FL algorithm with quantization heterogeneity.
It operates over $T$ learning rounds. 
Our approach is designed to ensure that even if an adversary intercepts the model updates, it cannot reconstruct the local training data. 
{The algorithm is presented in Algorithm~\ref{algorithm1}.}
%%%%

%%
At the $t^{\text{th}}$ learning round, the FC randomly selects devices to form $\{\cC_m\}_{m=1}^{M}$\footnote{The sets $\{\cC_m\}_{m=1}^{M}$ vary with each FL round $t$. For the ease of exposition, we omit $t$.} and sends the global model $\bw^t \in \R^{d \times 1}$ to $\{\cC_m\}_{m=1}^{M}$ as described in Steps \ref{step3} and \ref{step4} of Algorithm~\ref{algorithm1}. 
Then, each device $d_{m,u} \in \cC_m$ updates $\bw^t$ using its local data $\cD_{m,u}$ through $L$ iterations of mini-batch stochastic gradient descent (SGD) with the learning rate $\eta_t$, resulting in $\bw^{L,m,u,t}$ (Steps \ref{step6}-\ref{step9}).  
Herein, the mini-batch stochastic gradient $\Wnabla f_{m,u}(\bw^{l,m,u,t})$, based on a mini-batch $\cB_{l,m,u,t} \subseteq \cD_{m,u}$, is obtained by $\Wnabla f_{m,u}(\bw^{l,m,u,t}) = \frac{1}{|\cB_{l,m,u,t}|}\sum_{x\in \cB_{l,m,u,t}} \nabla \mathcal{\ell}(\bw^{l,m,u,t},x)$\cite{li2014MB}.
Each device computes the model difference $ \bv^{m,u,t} = \bw^{L,m,u,t}-\bw^t$ (Step \ref{step10}) and applies $\ell_1$-norm clipping to obtain $\widehat{\bv}^{m,u,t}$ with a clipping constant $C$ (Step \ref{step11}), stabilizing the training process \cite{Abadi2016}. 
In Step \ref{step12}, $\widehat{\bv}^{m,u,t}$ is quantized using the proposed SQ in \eqref{eq4} leading to $\widehat{\bw}^{m,u,t} = \cQ_{b_m,\epsilon_{1,m,u}}(\widehat{\bv}^{m,u,t})$, ensuring $\epsilon_{1,m,u}$-DP guarantee. 
Applying Lemma~\ref{lm:composition1} to the proposed SQ $\cQ_{b_m, \epsilon_{1,m,u}}(\cdot)$, the quantized output $\widehat{\bw}^{m,u,t} \in \R^{d \times 1}$ in Step \ref{step12} satisfies $\epsilon_{m,u}$-DP with
\vspace{-2mm}
\begin{equation}\label{eq:overallDP} 
\epsilon_{m,u} = d\epsilon_{1,m,u}. 
\vspace{-2mm}
\end{equation}
Next, the local update $\widehat{\bw}^{m,u,t}$ is sent to the FC for model fusion (Step \ref{step13}). 
The FC receives $(\widehat{\bw}^{m,u,t} + \bn^{m,u,t}) \in \R^{d \times 1}$ from device $d_{m,u}$, where $\bn^{m,u,t} \sim \cN(\mathbf{0},\sigma_{m,u}^2\bI)$ is additive white Gaussian noise on the link  between $d_{m,u}$ and the FC \cite{CongShen2022}.
Then, the FC updates the global model to {\small $\bw^{t+1} = \bw^t + \sumM \sum_{d_{m,u} \in \cC_m} \omega_{m,u}(\widehat{\bw}^{m,u,t}+\bn^{m,u,t})$} as in Step \ref{step15}, where the fusion weights $\{\omega_{m,u}\}$ satisfy $\sumM \sum_{u \in \cC_m}\omega_{m,u} = 1$ and $\omega_{m,u} \geq 0, \forall d_{m,u} \in \cC_m$. 
The fusion weights $\{\omega_{m,n}\}$ are optimized in the next subsection.
%%
%The fusion weights design for Step~\ref{step15} of the algorithm is discussed in the next subsection.
%%
\subsection{Fusion Weight Design}\label{secFWOpt}
In this subsection, we optimize the fusion weights to account for quantization resolution heterogeneity and link noise.
We first introduce a useful inequality for our analysis. 
Given $\zeta_1, \dots, \zeta_n \in \R$ and $\bx_1,\dots,\bx_n \in \R^{d \times 1}$, the following holds
\begin{equation}\label{eqCS}
\scalemath{0.8}{\left \|\sum_{i=1}^{n} \zeta_i\bx_i \right \|^2_2 \leq \left(\sum_{i=1}^{n} \zeta_i^2\right)\left(\sum_{i=1}^{n}\|\bx_i\|^2_2\right).}
\end{equation}
The result in \eqref{eqCS} is obtained by using the Cauchy-Schwarz inequality: {\small$\|\sum_{i=1}^{n} \zeta_i\bx_i \|^2_2 = \sum_{i,j} \zeta_i\zeta_j \lL \bx_i,\bx_j \lR \leq \sum_{i,j} |\zeta_i||\zeta_j| \|\bx_i\|_2\|\bx_j\|_2 = (\sum_{i=1}^{n}|\zeta_i|\|\bx_i\|_2)^2 \leq \left(\sum_{i=1}^{n} \zeta_i^2\right)\left(\sum_{i=1}^{n}\|\bx_i\|^2_2\right)$.}
We now define the ideal aggregated global model as {\small $\bw^{t+1,\star} = \bw^t + \sumM \sum_{d_{m,u} \in \cC_m} \omega_{m,u}\widehat{\bv}^{m, u, t}$}, which excludes the effects of the SQ (Step \ref{step12}) and noise {\small $\{\bn^{m,u,t}\}$}.
Our objective is to design $\{\omega_{m,u}\}$ to minimize the expected fusion error {\small $\E[\|\bw^{t+1} - \bw^{t+1,\star}\|_2^2]$}, which can be expressed as {\small $\E[\|\bw^{t+1} \!-\! \bw^{t+1,\star}\|_2^2] = \E[\|\sumM \sum_{d_{m,u} \in \cC_m} \omega_{m,u}(\cQ_{b_m, \epsilon_{1,m,u}}(\widehat{\bv}^{m,u,t}) \!-\! \widehat{\bv}^{m,u,t}  + \bn^{m,u,t})\|_2^2] \leq N \sumM \sum_{d_{m,u} \in \cC_m}\E[\|\omega_{m,u}(\cQ_{b_m, \epsilon_{1,m,u}}(\widehat{\bv}^{m,u,t}) - \widehat{\bv}^{m,u,t}  + \bn^{m,u,t})\|_2^2],$} where the last bound follows from \eqref{eqCS}.
From $\scalemath{0.9}{\E[\bn^{m,u,t}] =\mathbf{0}}$ and the fact that $\scalemath{0.9}{\{\bn^{m,u,t}\}}$ and $\scalemath{0.9}{\{\cQ_{b_m, \epsilon_{1,m,u}}(\widehat{\bv}^{m,u,t}) - \widehat{\bv}^{m,u,t}\}}$ are independent, we have $\E[\|\omega_{m,u}(\cQ_{b_m, \epsilon_{1,m,u}}(\widehat{\bv}^{m,u,t}) - \widehat{\bv}^{m,u,t}  + \bn^{m,u,t})\|_2^2] = \omega^2_{m,u}\big(\E\big[ \|\cQ_{b_m, \epsilon_{1,m,u}}(\widehat{\bv}^{m,u,t}) \!-\! \widehat{\bv}^{m,u,t}\|_2^2 \big]\! +\!  \E\big[ \|\bn^{m,u,t}\|^2 \big]\big)$. Thus, the following holds \looseness=-1
\begin{equation}\label{eq:feBound}
\begin{aligned}
    &\d4 \scalemath{0.9}{\E \left[ \left\|\bw^{t+1} - \bw^{t+1,\star}\right\|_2^2 \right] } \\ 
    &\d4\scalemath{0.9}{ \leq \! N \!\! \sumM \! \sum_{d_{m,u} \in \cC_m}\!\!\!\!\!\omega^2_{m,u} \left( \E \left[ \|\cQ_{b_m, \epsilon_{1,m,u}}(\widehat{\bv}^{m,u,t}) \!-\! \widehat{\bv}^{m,u,t}\|_2^2 \right] \! +\!  d\sigma^2_{m,u} \right) }. 
\end{aligned}
\end{equation} 
Defining the effective signal-to-noise ratio (SNR) of $d_{m,u} \in \cC_m$ as $\theta_{m,u} = \frac{1}{\E[\|\cQ_{b_m, \epsilon_{1,m,u}}(\widehat{\bv}^{m,u,t}) - \widehat{\bv}^{m,u,t}\|^2_2] + d\sigma^2_{m,u}}$, reflecting the combined effects of quantization distortion and noise power, the fusion weights are optimized, in our approach, by minimizing the upper bound in \eqref{eq:feBound}, \looseness=-1
\begin{subequations}\small\label{eq:fusionOpt2}
    \beq
    \{\omega_{m,u}^{\star}\}&=& \argmin_{\{\omega_{m,u}\}} \sumM \sum_{d_{m,u} \in \cC_m} \frac{\omega_{m,u}^2}{\theta_{m,u}}, \label{eq:fusionOpt2a}\\ 
    \text{subject to} && \sumM \sum_{d_{m,u} \in \cC_m}  \omega_{m,u} = 1,\label{eq:fusionOpt2b} \\ 
    && \omega_{m,u} \geq 0, \forall d_{m,u}.\label{eq:fusionOpt2c}
    \eeq 
\end{subequations}
%%
%The problem in \eqref{eq:fusionOpt2} is convex, existing an optimal solution.
%%
Applying the Cauchy-Schwarz inequality to the constraint in \eqref{eq:fusionOpt2b} yields $1 = \big( \sumM \sum_{d_{m,u} \in \cC_m}  \omega_{m,u} \big)^2 \leq \big(\sumM  \! \sum_{d_{m,u} \in \cC_m} \! \frac{\omega_{m,u}^2}{\theta_{m,u}}\big)\big(\sumM \! \sum_{d_{m,u} \in \cC_m} \! \theta_{m,u}\big)$, where the  equality holds if and only if \looseness=-1
\begin{equation}\label{eq:fusionWOpt} \omega_{m,u} = \frac{\theta_{m,u}}{\sum_{m}\sum_{d_{m,u'}\in \cC_{m}} \theta_{m,u'}}, \end{equation}
which is the optimal solution to \eqref{eq:fusionOpt2}.
The solution in \eqref{eq:fusionWOpt} implies that the fusion weights $\{\omega_{m,u}\}$ are assigned proportionally to the effective SNR $\theta_{m,u}$. 
It assigns larger fusion weights to devices with larger $\theta_{m,u}$, which contribute more reliable model updates, and vice versa. 
In contrast to prior fusion weight designs in \cite{CongShen2021, Yujia2022}, the proposed fusion weights in \eqref{eq:fusionWOpt} account for the combined effects of heterogeneous quantization resolution, DP requirements, and link noise between the FC and devices.
%%

% %%
% Next, we leverage the fusion approach detailed in Section~\ref{SecIIA} to assign fusion weights $\{\omega_{m,u}\}$ that encounter the combined effect of quantization heterogeneity $\{\cQ_{b_m}(\cdot)\}_{m=1}^{M}$ and channel noise $\{\bn^{m,u,t}\}$, $\forall d_{m,u} \in \cC_m$.
% %%
% %%
% It is straightforward to observe that applying Lemma~\ref{lm:fusionLM} to $\omega_{m,u}$ leads to 
% %%

\section{Convergence Analysis and Cluster Sizes Optimization}\label{SecV}
In this section, we derive an upper bound on the learning error of Algorithm~\ref{algorithm1}.
Our analysis follows the conventional framework established in prior works \cite{Hassani2019, Chang2021}, with a key distinction: we explicitly characterize the combined effect of the quantization resolution heterogeneity, link noise, and cluster sizes on the convergence behavior of the proposed algorithm.
We then optimize the cluster sizes $\{c_m\}_{m=1}^{M}$ to minimize the learning error upper bound.
\subsection{Learning Error Bound Analysis}
We adopt standard assumptions from the FL literature \cite{Hassani2019}.\looseness=-1
\begin{assumption}\label{assumption1}($\gamma$-smooth \cite{Nesterov2014}) 
    The local loss function $f_{m,u}(\cdot)$ is convex, differentiable, and $\gamma$-smooth ($\gamma >0$ is a fixed  constant), satisfying $\|\nabla f_{m,u}(\bx) - \nabla f_{m,u}(\by)\|_2 \leq \gamma \|\bx-\by\|_2$, $\forall d_{m, u}$ and $\forall \bx, \by \in \R^{d \times 1}$. 
    It equivalently holds that $f_{m,u}(\by) \leq f_{m,u}(\bx) + \lL \nabla f_{m,u}(\bx),\by-\bx \lR + \frac{\gamma}{2}\|\by-\bx\|^2_2$.
\end{assumption}
\begin{assumption}\label{assumption2} ($\mu$-strong convex \cite{Nesterov2014}) The local loss function $f_{m,u}(\cdot)$ is $\mu$-strong convex ($\mu > 0$ is a fixed constant), satisfying $\lL \nabla f_{m,u}(\bx) - \nabla f_{m,u} (\by),\bx-\by \lR \geq \mu\|\bx-\by\|^2_2$,  $\forall d_{m, u}$ and $\forall \bx, \by \in \R^{d \times 1}$.
\end{assumption}
\begin{assumption}\label{assumption3} (Unbiasedness and boundedness \cite{Hassani2019}) 
    The stochastic gradient at each device is an unbiased estimator of the global gradient, i.e., $\E[\wnabla f_{m,u}(\bx)] = \nabla f(\bx)$, with a bounded variance $\E[\|\wnabla f_{m,u}(\bx)-\nabla f(\bx)\|^2_2] \leq \zeta$, $\forall d_{m,u}$, where $\zeta > 0$ is a fixed constant.
\end{assumption}
\begin{assumption}\label{assumption4} (Clipping error boundedness \cite{koloskova2023revisiting}) 
Given $\bx \in \R^{d \times 1}$ and a positive constant $C$, the clipping error is $b(\bx,C) = \big\|\bx - \min\big\{1, \frac{C}{\|\bx\|_{1}}\big\} \bx \big\|^2_2$. 
Then, the clipping error is bounded by $b(\bx,C) \leq \Lambda\|\bx\|_2^2$, where $\Lambda >0$ is a fixed constant.
\end{assumption}
The bias introduced by the clipping in Step \ref{step11} and the SQ in Step \ref{step12} of Algorithm~\ref{algorithm1} makes the convergence analysis of the proposed FL difficult \cite{demidovich2023guide, beznosikov2023biased}. 
To regulate such a bias, we adopt the following assumption.
\begin{assumption}\label{assumption5} (Bounded bias \cite{demidovich2023guide}) 
Suppose $\ddot{\bw} \in \R^{d \times 1}$ is a biased estimate of a model $\bw \in \R^{d \times 1}$. 
Then, there exists a positive constant $\vartheta$ such that $\E[\lL \ddot{\bw} - \bw, \bw - \wst \lR] \leq \vartheta$, where $(\ddot{\bw} - \bw)$ is the bias and $(\bw - \wst)$ is the sub-optimality gap.
\end{assumption}
\noindent 
Assumption~\ref{assumption5} quantifies the cross-correlation between the bias and sub-optimality gap.
We now derive an upper bound on the expected learning error $\E[\|\bw^{t+1} - \bw^{\star}\|^2_2]$ of Algorithm~\ref{algorithm1}. 
We begin with establishing an intermediate inequality, which is useful for our derivations. 
Define ${\whw}^{t+1} = \bw^t + \sumM \sum_{d_{m,u} \in \cC_m}\omega_{m,u}\widehat{\bw}^{m,u,t}$ and $\bbw^{t+1} = \sumM \sum_{d_{m,u}\in \cC_m} \omega_{m,u}\bw^{L,m,u,t}$. 
It is not difficult to observe that $\overline{\bw}^{t+1} = \bw^t + \sumM\sum_{d_{m,u} \in \cC_m}\omega_{m,u}\bv^{m,u,t}$. 
Thus, ${\whw}^{t+1}$ is a biased estimate of $\overline{\bw}^{t+1}$ because $\widehat{\bw}^{m,u,t}$ is obtained by quantizing and clipping $\bv^{m,u,t}$ (Steps \ref{step11} and \ref{step12}).
We have the following lemma about the expected learning error. \looseness = -1
\begin{lemma}\label{lm:equality1}
For the $(t+1)^{\text{th}}$ FL round, the expected learning error $\E[\|\bw^{t+1} - \bw^{\star}\|^2_2]$ satisfies  
\begin{multline} 
\label{eq:ub1}
\d4\E[\|\bw^{t+1}-\bw^{\star}\|^2_2] \leq  \E[\|\bw^{t+1}-\whw^{t+1}\|^2_2] +  \E[\|\whw^{t+1}-\bbw^{t+1}\|^2_2]\\+\E[\|\bbw^{t+1}-\wst\|^2_2] + 2\vartheta,
\end{multline}
where $\vartheta$ is defined in Assumption~\ref{assumption5}. 
\begin{proof}
    From Step \ref{step15} of Algorithm~\ref{algorithm1} and noting that $\E[\bn^{m,u,t}] = \mathbf{0}, \forall d_{m,u} \in \cC_m$, the following holds $\E_{\{\bn^{m,u,t}\}}[\bw^{t+1}] = \bw^{t} + \sumM \sum_{d_{m,u} \in \cC_m}\omega_{m,u}\widehat{\bw}^{m,u,t}$. Thus, we have
\begin{equation}
\begin{aligned}
    \E_{\{\bn^{m,u,t}\}}[\bw^{t+1}] = \whw^{t+1}.\label{eq:lmeq1d}
\end{aligned}
\end{equation}%%
Next, $\E[\|\bw^{t+1}-\bw^{\star}\|^2_2] = \E[\|\bw^{t+1}\!-\!\whw^{t+1}\|^2_2] \!+\! 2 \E \left[\lL \bw^{t+1}\!-\!\whw^{t+1}, \whw^{t+1} -\wst\lR\right] \!+\! \E[\|\whw^{t+1} -\wst\|^2_2] = \E[\|\bw^{t+1}\!-\!\whw^{t+1}\|^2_2] + \! \E[\|\whw^{t+1} -\wst\|^2_2]$, where the last equality follows from the equality $\E_{\{\bn^{m,u,t}\}} \big[\lL \bw^{t+1}\!-\!\whw^{t+1}, \whw^{t+1} -\wst\lR\big] = \lL \E_{\{\bn^{m,u,t}\}} \big[\bw^{t+1}\!-\!\whw^{t+1}\big], \whw^{t+1} -\wst\lR$ by also noting  \eqref{eq:lmeq1d}. 
Then, $\E[\|\bw^{t+1}-\bw^{\star}\|^2_2] = \E[\|\bw^{t+1}-\whw^{t+1}\|^2_2] + \E[\|\whw^{t+1}-\bbw^{t+1}\|^2_2] + 2\E[\lL \whw^{t+1}-\bbw^{t+1}, \bbw^{t+1}-\wst \lR] +\E[\|\bbw^{t+1}-\wst\|^2_2].$
Therefore, the following holds \sloppy \looseness=-1
\begin{equation}\label{eq:equality1b}
\begin{aligned}
\d4\d4\d4\d4\d4&\E[\|\bw^{t+1}-\bw^{\star}\|^2_2] \leq \E[\|\bw^{t+1}-\whw^{t+1}\|^2_2]  + \\
&~~~~~~~~\E[\|\whw^{t+1}-\bbw^{t+1}\|^2_2]+\E[\|\bbw^{t+1}-\wst\|^2_2] + 2\vartheta,
\end{aligned}
\end{equation}
where \eqref{eq:equality1b} is due to Assumption~\ref{assumption5} and the fact that $\whw^{t+1}$ is a biased estimate of $\overline{\bw}^{t+1}$.
This completes the proof.
\end{proof}
\end{lemma}
From Lemma~\ref{lm:equality1}, we proceed to derive upper bounds for $\E[\|\whw^{t+1}-\bbw^{t+1}\|^2_2]$, $\E[\|\bw^{t+1}-\whw^{t+1}\|^2_2]$, and $\E[\|\bbw^{t+1}-\wst\|^2_2]$ on the right-hand-side (r.h.s) of \eqref{eq:ub1}.
To this end, we define a sequence of vectors $\{\bb^{l,t}\}_{l=0}^{L}$ \cite{Hassani2019} using gradient descent iterations in terms of the gradient of the global loss function $f$ defined in \eqref{eq1} with $\bb^{0,t} = \bw^{t}$,
\begin{equation}
\label{eq:defb}
     \bb^{l+1,t} =\bb^{l,t} - \eta_t\nabla f(\bb^{l,t}), \text{for } l =0,\dots,L-1. 
\end{equation} where $\eta_t$ is the learning rate in Step \ref{step8} of Algorithm~\ref{algorithm1}.
We also define intermediate variables: gradient error {\small $\be^{l,t} = \sumM\sum_{d_{m,u} \in \cC_m}\omega_{m,u} [\widetilde{\nabla}f_{m,u}(\bw^{l,m,u,t}) - \nabla f(\bb^{l,t})]$}, cumulative gradient error $\be^{t} = \sum_{l=0}^{L-1}\be^{l,t}$,  gradient sum $\bg^t = \sum_{l=0}^{L-1}\nabla f(\bb^{l,t})$, and deviation term\sloppy
\begin{equation}\label{eq:defat}
    a^{l,t} = \frac{1}{K}\sumM\sum_{d_{m,u} \in \cG_m} \E[\| \bw^{l,m,u,t} -  \bb^{l,t}\|^2_2].  
\end{equation}
Given the definitions in \eqref{eq:defb} and \eqref{eq:defat}, and for $\eta_t \leq \min\{\frac{1}{\mu},\frac{\mu}{\gamma^2}\}$, we have \cite{Hassani2019}, 
\begin{subequations}
    \beq
    \|\bb^{l,t} - \bw^{\star}\|^2_2  & \leq & (1-\mu \eta_t)^l\|\bw^{t} - \bw^{\star}\|^2_2, \label{T1_eq6} \\  
     a^{l,t} &\leq& \eta_t^2 L \zeta (1+L\eta^2_t\gamma^2)^{l-1}.  \label{T1_eq10}   
    \eeq 
\end{subequations}
Since a part of the following derivation is inspired by prior works \cite{Hassani2019}, we present the features that are unique and refined in this work, while relegating those standard derivations to the supplement document in \cite{Qua2025Supp}. 
We now readily bound the error term $\E \big[ \| \whw^{t+1} - \overline{\bw}^{t+1} \|^2_2 \big]$ on the r.h.s of \eqref{eq:ub1}.
\begin{lemma}\label{lm:2ndTerm}Suppose Assumptions~\ref{assumption1}-\ref{assumption4} hold and the learning rate $\eta_t \leq \min\{\frac{1}{\mu}, \frac{1}{L\gamma}, \frac{\mu}{\gamma^2}\}$, where $\gamma$ and $\mu$ are defined in Assumption~\ref{assumption1} (smoothness) and Assumption~\ref{assumption2} (strong convexity),  respectively. 
Then, the following holds
\begin{multline}\small
\label{eq:ub3}
\d4\!\! \E \big[ \| \whw^{t+1} - \overline{\bw}^{t+1} \|^2_2 \big] \leq  4\Lambda N \eta_t^2  L^2\gamma^2 \E[\|\bw^t - \bw^{\star}\|^2_2]\\\d4\d4~~~~~ + 8 d C^2 \!\!\sumM\!\frac{c_m}{(2^{b_m}-1)^2} \! + \! 4\Lambda K \eta_t^2(L\zeta \!+\! L^3\gamma^2\zeta\eta_t^2e),\!\!\!\!\!
\end{multline}where $\zeta$ is  defined in Assumption~\ref{assumption3} (bounded variance) and $\Lambda$ is defined in Assumption~\ref{assumption4} (clipping error boundedness).
\end{lemma}
\begin{proof}
    See Appendix~\ref{AppTerm2}.
\end{proof}%%
\noindent  The term $\E [ \| \whw^{t+1} - \overline{\bw}^{t+1} \|^2_2]$ represents the expected $\ell_2$-norm of bias caused by the clipping (Step \ref{step11}) and stochastic quantization (Step \ref{step12}).
The upper bound for the expected $\ell_2$-norm of bias in \eqref{eq:ub3} consists of three terms. 
The first term captures the impact of the learning error $\E[\|\bw^{t} - \bw^{\star}\|^2_2]$ from the preceding FL round. 
The second term represents the upper bound of quantization distortion, which tends to $0$ as $b_m \rightarrow \infty$, $\forall m$. 
The last term reflects the impact of SGD variance, which decreases as the learning rate $\eta_t$ decreases.
For the term $\E[\|\bw^{t+1}-\whw^{t+1}\|^2_2]$ on the r.h.s of \eqref{eq:ub1}, we have the following lemma.
\begin{lemma}\label{lm:3rdTerm} For the $(t + 1)^{\text{th}}$ FL round, the following holds
\begin{equation}
    \label{eq:ub4}
\begin{aligned} \E[\|\bw^{t+1}-\whw^{t+1}\|^2_2] \leq d\sumM \sum_{d_{m,u} \in \cG_m} \frac{c_m}{g_m}\sigma_{m,u}^2.
\end{aligned}
\end{equation}
\end{lemma}
\begin{proof}
    See Appendix~\ref{AppTerm3}.
\end{proof}%%
\noindent The term on the r.h.s. of \eqref{eq:ub4} accounts for link noise, which disappears if the device-to-FC links are noiseless.
We now bound the error term $\E[ \| \overline{\bw}^{t+1} - \bw^{\star}\|^2_2]$ on the r.h.s of  \eqref{eq:ub1}. Since this term is independent of the clipping, quantization, and link noise,  bounding it follows the conventional framework \cite{Hassani2019}. 
Its detailed proof is provided in the supplement document \cite{Qua2025Supp}.
\begin{lemma}\label{lm:1stTerm}
Suppose Assumptions~\ref{assumption1}-\ref{assumption3} hold and the learning rate $\eta_t \leq \min\{\frac{1}{\mu}, \frac{1}{L\gamma}, \frac{\mu}{\gamma^2}\}$, where $\gamma$ and $\mu$ are defined in Assumption~\ref{assumption1} (smoothness) and Assumption~\ref{assumption2} (strong convexity),  respectively. 
Then, the following holds 
\begin{multline}\small
\label{eq:ub2}
\!\!\d4 \E[ \| \overline{\bw}^{t+1} - \bw^{\star}\|^2_2] \leq (1+K\eta_t^2)(1-\mu \eta_t)^L\E[\|\bw^{t} - \bw^{\star}\|^2_2] \\  + \Big(\frac{1}{K} + \eta_t^2\Big)(L^2\zeta + KL^3\gamma^2\zeta\eta^2_t e),
\end{multline}where $\zeta$ is defined in Assumption~\ref{assumption3} (bounded variance).  
\end{lemma}
%%
% \begin{proof}
%   See the supplement document \cite{Qua2025Supp}.\looseness=-1
% \end{proof}%%
%%

%%
%\noindent This term on the r.h.s of \eqref{eq:ub2} diminishes as the number of local SGD iterations $L \to \infty$ due to the fact that $0 \leq 1 - \mu\eta_t < 1$.
%%
%This implies that increasing local SGD iterations $L$ improves convergence. 
%%
%The second term represents the effect of stochastic gradient variance, which decreases as the number of participating devices $K$ increases and the learning rate $\eta_t$ decreases. 
%%

%%
 \begin{figure*}[htb]
        \centering
        \subfloat[]{%
        \includegraphics[width=0.38\textwidth]{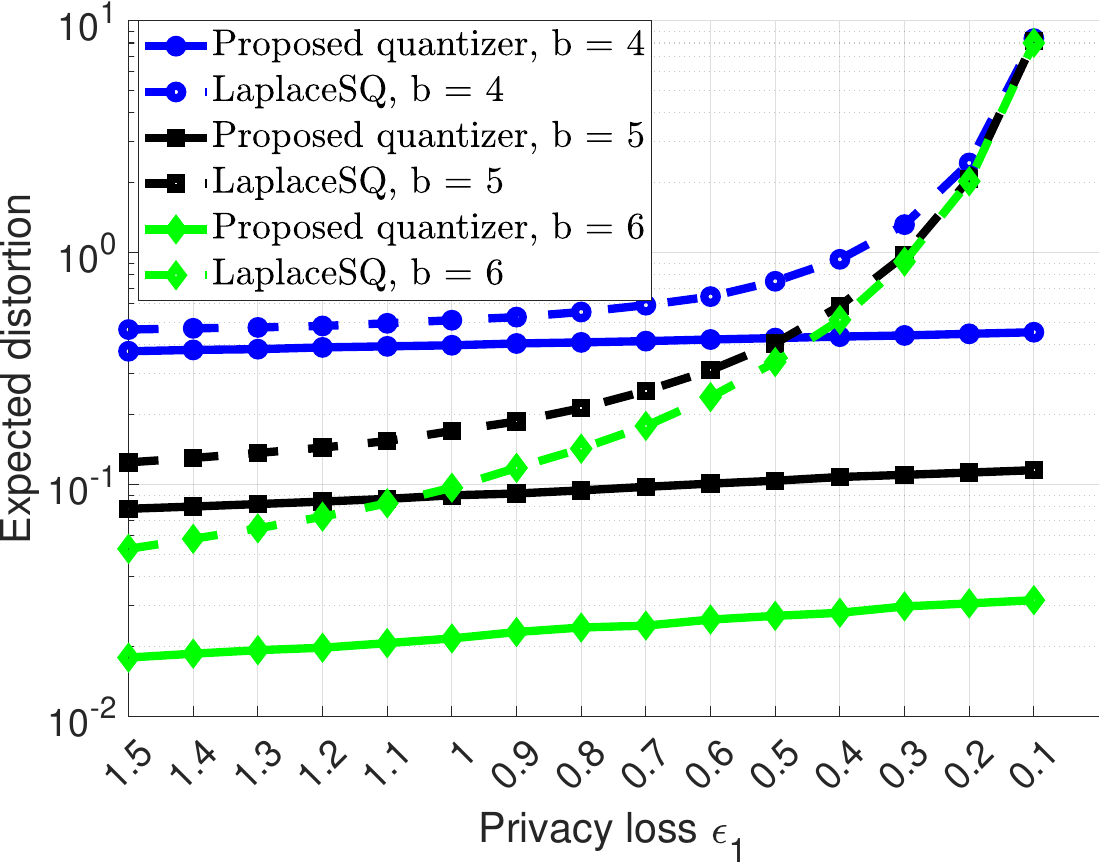}%
     \label{fig:SQ_Lap}
    }%
    \hspace{2mm}
        \subfloat[]{%
        \includegraphics[width=0.28\textwidth]{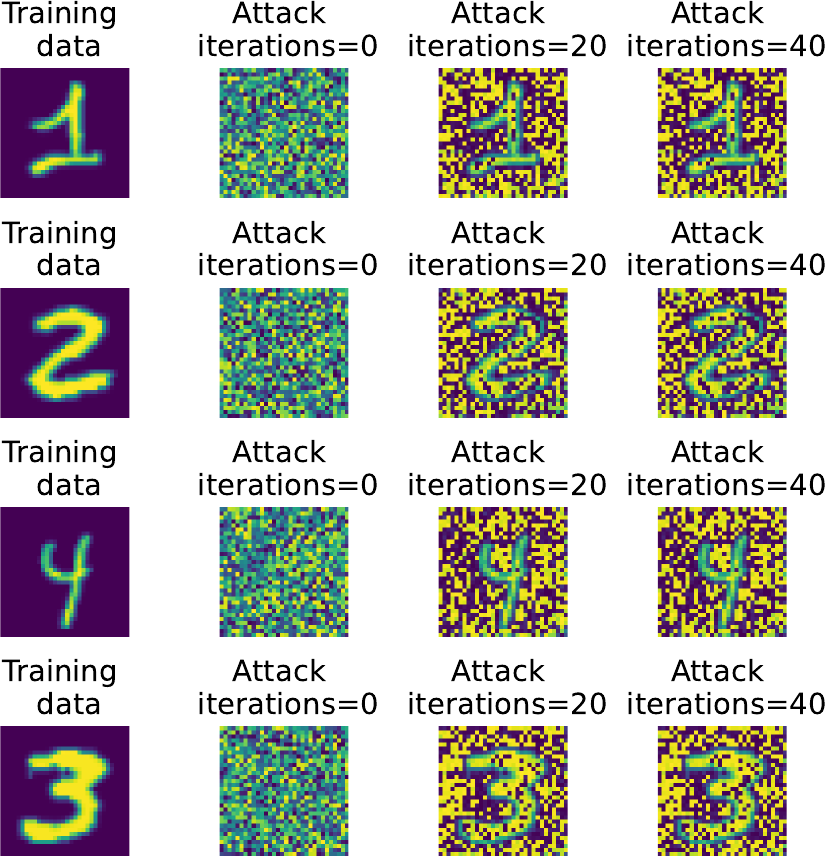}%
        \label{fig:5.5a}%
        }%
    \hspace{2mm}
      %  \hfill
        \subfloat[]{%
        \includegraphics[width=0.28\textwidth]{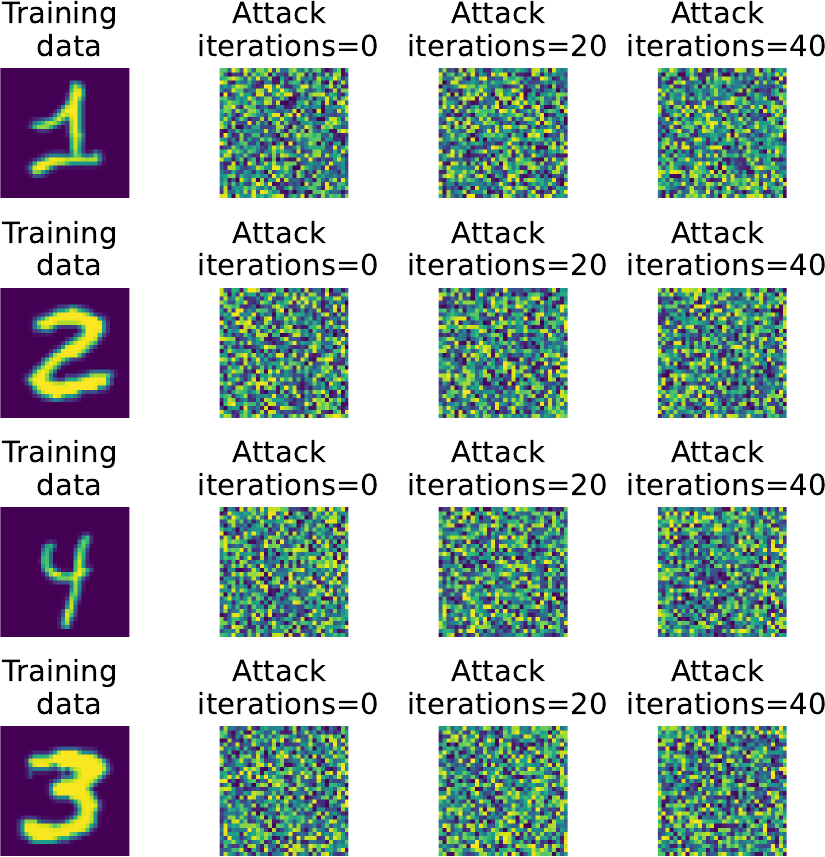}%
     \label{fig:5.5b}
    }%
    
%       \vspace{-2mm}
        \caption{a) Expected distortions of the proposed quantizer $\cQ_{b,\epsilon_1}(a)$ and the LaplaceSQ mechanism  $\mathsf{L}(a)$ with $b \in \{4,5,6\}$ when $\epsilon_1$ decreases from $1.5$ to $0.1$. b) and c) DLG attack results on SQ-FL and Algorithm 1, respectively, for $\epsilon_{1,m,u} = 10^{-6}$ after attack iterations in $\{0,20,40\}$.}
        \label{fig:5.5}
    \vspace{-2mm}
    \end{figure*}

% %%
% From \eqref{eq:ub4}, the term $\E[\|\bw^{t+1}-\whw^{t+1}\|^2_2]$ vanishes when assuming that device-to-FC channels are noiseless. 
% %%
%%
Using \eqref{eq:ub1} and the results in \eqref{eq:ub4} and \eqref{eq:ub2}, the expected learning error $\E[\|\bw^{t+1} - \bw^{\star}\|^2_2]$ is upper bounded as follows:
\begin{equation}\small\label{eq:ub5}
\begin{aligned}
    \d4 \E[\|\bw^{t+1}-\bw^{\star}\|^2_2]  \leq  m_{1,t}\E[\|\bw^t - \bw^{\star}\|^2_2] + m_{2,t}  + \Delta ,
\end{aligned}
\end{equation} where $m_{1,t} = (1+K\eta_t^2)(1-\mu \eta_t)^L + 4 \eta_t^2 \Lambda N L^2\gamma^2 $, $m_{2,t} = \eta_t^2 (L^2\zeta + L^3 \gamma^2\zeta e + 4\Lambda K L \zeta) + \eta^4_t(1+4\Lambda)KL^3\gamma^2\zeta e$, and 
\begin{equation}\small
    \Delta \!=\! \frac{L^2\zeta}{K} + d\sumM \sum_{d_{m,u} \in \cG_m} \! \frac{c_m}{g_m}\!\Big(\frac{8C^2}{(2^{b_m}-1)^2} + \sigma_{m,u}^2\Big) + 2\vartheta. \label{eq:Delta}
\end{equation}
\noindent The deviation term $\Delta$ in \eqref{eq:Delta} captures the combined effects of cluster sizes, link noise, quantization resolution heterogeneity, and quantization error. 
From the bound in \eqref{eq:ub5}, the following theorem describes the convergence behavior of the proposed Algorithm~\ref{algorithm1}.
\begin{theorem}
\label{theorem1} 
Suppose the learning rate $\eta_t = \frac{8}{L \mu  t}$.
Substituting $\eta_t = \frac{8}{L\mu t}$ into $m_{2,t}$ leads to $m_{2,t} = \frac{k_1}{t^2} + \frac{k_2}{t^4}$, where $k_1 = \frac{64}{L^2\mu^2}\Big(L^2\zeta + L^3\gamma^2\zeta e + 4 \Lambda KL \zeta\Big)$ and $k_2 = \frac{8^4}{L^4\mu^4}(1+4\Lambda)K L^3 \gamma^2 \zeta e$. 
Let $t_0 \geq 4$ be the first FL round such that $\eta_{t_0} \leq \min\{ \frac{1}{\mu}, \frac{\mu}{\gamma^2},\frac{1}{L\mu}, \frac{1}{L\gamma}, \frac{L\mu}{2(L^2\mu^2 + K + 4\Lambda N L^2\gamma^2)}\}$, where $\gamma$, $\mu$, and $\Lambda$ are defined in Assumptions~\ref{assumption1}, \ref{assumption2}, and \ref{assumption4},  respectively.
Then, the expected learning error $\E[\|\bw^{t}-\bw^{\star}\|^2_2]$ at the $t^{\text{th}}$ round is upper bounded by \sloppy
    \begin{equation}
        \label{eqtheorem1}
        \E[\|\bw^{t}-\bw^{\star}\|^2_2] \leq         \frac{t_0}{t}\E[\|\bw^{t_0}-\bw^{\star}\|^2_2] + m_{3,t} +  (t-4)\Delta,
    \end{equation} where $m_{3,t} = \frac{k_1}{t} + \frac{k_2}{t^3}$ and $\Delta$ is defined in \eqref{eq:Delta}.
\end{theorem}
\begin{proof}
    See Appendix~\ref{AppTheorem1}.
\end{proof}
As $t$ increases, the first and second terms on the r.h.s of \eqref{eqtheorem1} diminish, ensuring that learning utility improves over time. 
However, the third term, $(t-4)\Delta$ grows unbounded as $t$ increases. 
This observation underscores a fundamental trade-off: initial extended training enhances model accuracy, but continued training  can negatively impact learning utility.
This finding is consistent with prior works in privacy-preserving FL \cite{Abadi2016, Kang2020, Kang2023, Qua2023Asilomar}.
\subsection{Cluster Sizes Optimization}\label{SecVB}
As $t$ increases in \eqref{eqtheorem1}, it is advantageous to minimize $\Delta$ to mitigate learning utility degradation.
It is noteworthy that $\Delta$ in \eqref{eq:Delta} exhibits a linear relationship with cluster sizes $\{ c_m \}$, which motivates the cluster size optimization.
The FC optimizes the cluster sizes $\{c_m\}_{m=1}^{M}$ to minimize the deviation term $\Delta$ in \eqref{eqtheorem1} leading to 
\begin{subequations}\small
    \label{eq:op1}
    \beq
    \d4\min_{\{c_m\}_{m=1}^M} \d4&& \d4\sumM  \frac{c_m}{g_m}\left(\sum_{d_{m,u} \in \cG_m} \frac{8C^2}{(2^{b_m}-1)^2} + \sigma_{m,u}^2\right), \label{eq:op1.1}\\
    \d4\text{subject to } \d4&&\d4 \sumM b_m c_m \leq B,    \label{eq:op1a}\\
    \d4&&\d4 1 \leq c_m \leq g_m, \forall m,    \label{eq:op1b}\\
    \d4&&\d4 \sumM c_m = N,    \label{eq:op1c}
    \eeq 
\end{subequations}
where \eqref{eq:op1a} constrains the total number of quantization bits in each FL round defined in \eqref{eq:defB}, \eqref{eq:op1b} limits the number of devices in each cluster $c_m$ to be bounded by $g_m$, and \eqref{eq:op1c} ensures that the total number of selected devices in each FL round is $N$.
The problem in \eqref{eq:op1} is linear integer programming (LIP).
Thus, the optimal solution to the problem in  \eqref{eq:op1} can be obtained using the branch-reduce-and-bound algorithm \cite{li2006nonlinear, TJ2015}.
Since it is a standard procedure, we omit the detail.

\section{Simulation Results}\label{SecVI}
{
This section presents comprehensive numerical results to validate the effectiveness of the proposed framework. 
We begin by evaluating the expected quantization distortion introduced by the proposed differentially private SQ in Section~\ref{SecII.B}.
Next, we assess the privacy protection capability of Algorithm~\ref{algorithm1} against a model inversion attack.
Then, we examine the learning utility of Algorithm~\ref{algorithm1} under various DP requirements, focusing on training loss and test accuracy with different configurations.
}

\subsection{Quantization Distortion Evaluation}
{This subsection aims to evaluate the expected quantization distortion introduced by the proposed differentially private SQ in Lemma~\ref{lm1}, in comparison with the conventional LaplaceSQ mechanism in \eqref{eqLaplaceSQ}.}
Fig.~\ref{fig:SQ_Lap} demonstrates the expected distortions of {the proposed quantizer $\cQ_{b,\epsilon_1}(a)$ in Lemma~\ref{lm1} and the  LaplaceSQ mechanism $\mathsf{L}(a)$ in \eqref{eqLaplaceSQ}} for $b \in \{4,5,6\}$, as $\epsilon_1$ decreases from 1.5 to 0.1.
The input $a$ is drawn uniformly from the interval $[-10,10]$, leading to the $\ell_1$-sensitivity (in \eqref{eq:l1sen}) $\rho = 20$ of the LaplaceSQ mechanism.
It can be seen from Fig.~\ref{fig:SQ_Lap} that increasing the quantization bits $b$ from 4 to 6 reduces the expected distortion of the proposed quantizer and LaplaceSQ.
Notably, the proposed quantizer $\cQ_{b,\epsilon_1}(a)$ achieves lower expected distortion than the LaplaceSQ mechanism. 
When $\epsilon_1 = 0.1$ and $b  = 6$, the distortion of the proposed quantizer is approximately $2.5$ orders of magnitude smaller than the LaplaceSQ mechanism.
{These empirical observations confirm that the expected distortion of the proposed SQ is bounded while that of LaplaceSQ grows unbounded. This is consistent with our analysis in Remark~\ref{rmk1}.}
\subsection{Privacy Protection Evaluation}
We now numerically evaluate the privacy protection capability of Algorithm~\ref{algorithm1} against the model inversion attack \cite{Zhu2019}. 
The adversary in Fig.~\ref{fig:1} employs the deep leakage from gradients (DLG) attack \cite{Zhu2019} to reconstruct local training data from model updates. 
Following the setup in \cite{Zhu2019, Lang2023}, we use the LeNet \cite{LeNet1989} architecture for MNIST classification \cite{lecun2010} to assess the attacks.
We consider the worst-case scenario, where the adversary obtains a noiseless update $\widehat{\bw}^{m,u,t}$ (Step \ref{step13} of Algorithm~\ref{algorithm1}) from a device $d_{m,u} \in \cG_m$ using $b_m = 6$ bits. 
The stochastic gradient $\Wnabla f_{m,u}$ (Step \ref{step8}) is computed from a single data point with $L=1$ local iteration. 
This setting is intentionally chosen to be the most vulnerable scenario to a DLG attack because it is more challenging for DLG to reconstruct multiple training data points $(|\cB_{m,u,t}| > 1)$ simultaneously under the noisy links \cite{Zhu2019, Lang2023}.\looseness=-1
To perform the attack, the adversary initializes a dummy data point and updates it via gradient descent to minimize the difference between its gradient and $\widehat{\bw}^{m,u,t}$. 
Each gradient descent step is referred to as an attack iteration.
We assess reconstruction quality using the structural similarity index measure (SSIM) \cite{SSIM2004, Lang2023}, where a higher SSIM indicates greater data leakage.
To evaluate privacy protection capability under the DLG attack, we define a standard FL baseline using the stochastic quantizer from \cite{Alistarh2017}, which we denote as SQ-FL. 
SQ-FL is implemented by replacing the quantizer $\cQ_{b,\epsilon_1}(\cdot)$ in Step~\ref{step12} with the quantizer $Q_b(\cdot)$ in \cite{Alistarh2017}.
Figs.~\ref{fig:5.5a} and \ref{fig:5.5b} compare reconstruction results by the DLG attack on SQ-FL and Algorithm~\ref{algorithm1}, respectively, for $\epsilon_{1,m,u} = 10^{-6}$ at attack iterations 0, 20, and 40.
In Fig.~\ref{fig:5.5a}, DLG gradually reconstructs data under SQ-FL, revealing clear privacy leakage. 
In contrast, Fig.~\ref{fig:5.5b} shows severely distorted reconstructions, preventing the adversary from recovering meaningful information.
This difference arises because SQ-FL lacks a formal DP guarantee while Algorithm~\ref{algorithm1} ensures a DP guarantee with $\epsilon_{1,m,u} = 10^{-6}$.\looseness=-1
%%

%%
%  \vspace{-4mm}
   \begin{table}[htb]
    \centering
     \caption{Comparison of SSIM values between Algorithm~\ref{algorithm1} for $\epsilon_{1,m,u} = 10^{-6}$ and SQ-FL for attack iterations in $\{0,20,40\}$.}
     %\vspace{-2mm}
    \includegraphics[width=0.48\textwidth]{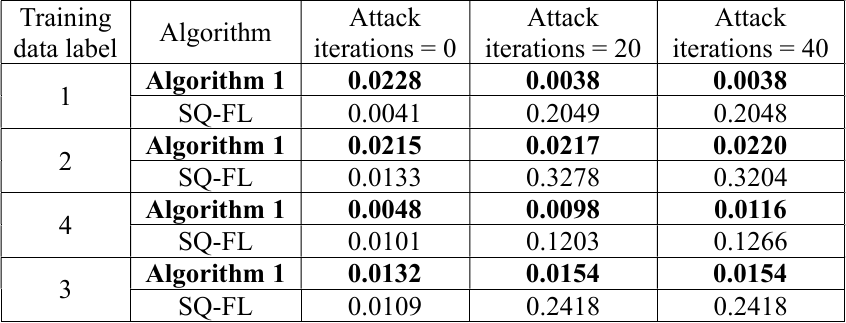}
    \label{tab:ssim_comparison}
    %\vspace{-2mm}
    \end{table}
Table~\ref{tab:ssim_comparison} provides a quantitative assessment of privacy leakage by comparing SSIM values of Algorithm~\ref{algorithm1} and SQ-FL for the reconstructed results shown in Figs.~\ref{fig:5.5a} and \ref{fig:5.5b}.
The results in Table~\ref{tab:ssim_comparison} confirm that Algorithm~\ref{algorithm1} maintains low SSIM values across all attack iterations. 
Conversely, SQ-FL exhibits significant privacy leakage, with SSIM values increasing as attack iterations progress. 
{It is worth noting that SQ-FL attains better learning utility (i.e., training loss and testing accuracy) than Algorithm~\ref{algorithm1} because the SQ $Q_b(\cdot)$ in \cite{Alistarh2017} yields a lower quantization distortion than the proposed SQ $\cQ_{b,\epsilon_1}(\cdot)$ in Lemma~\ref{lm1}.
However, to ensure the same DP guarantee as Algorithm~\ref{algorithm1}, SQ-FL would need to incorporate the Laplace mechanism described in Lemma~\ref{LpMechanism}.
The learning utility of the latter incorporation is evaluated in the next subsection.}
 \begin{figure*}[htb]
        \centering
        \subfloat[]{%
        \includegraphics[width=0.32\textwidth]{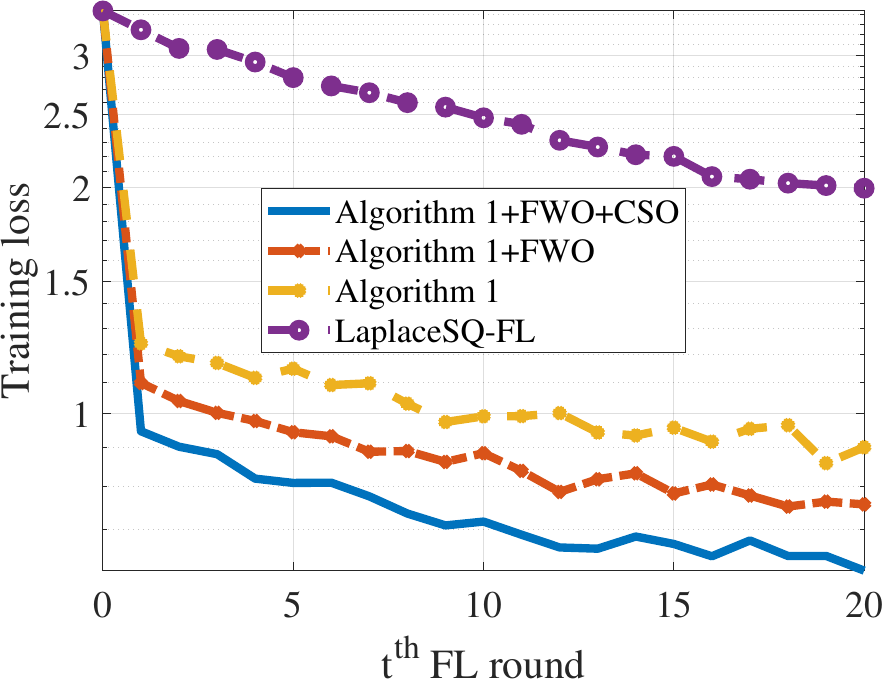}%
        \label{fig:5.1a}%
        }%
    \hspace{2mm}
        \centering
        \subfloat[]{%
        \includegraphics[width=0.32\textwidth]{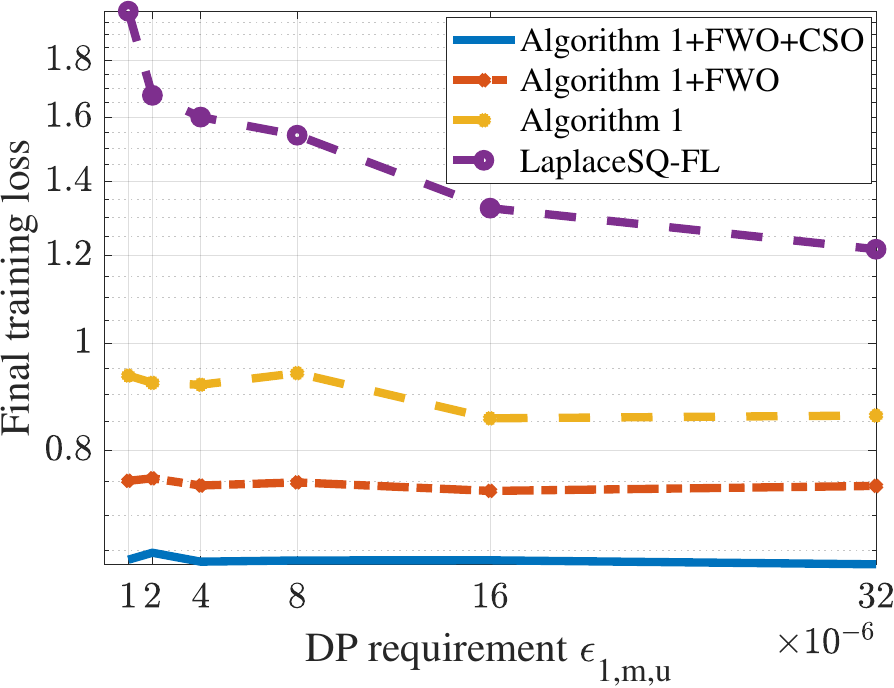}%
        \label{fig:5.2a}%
        }%
     \hspace{2mm}
        \centering
        \subfloat[]{%
        \includegraphics[width=0.32\textwidth]{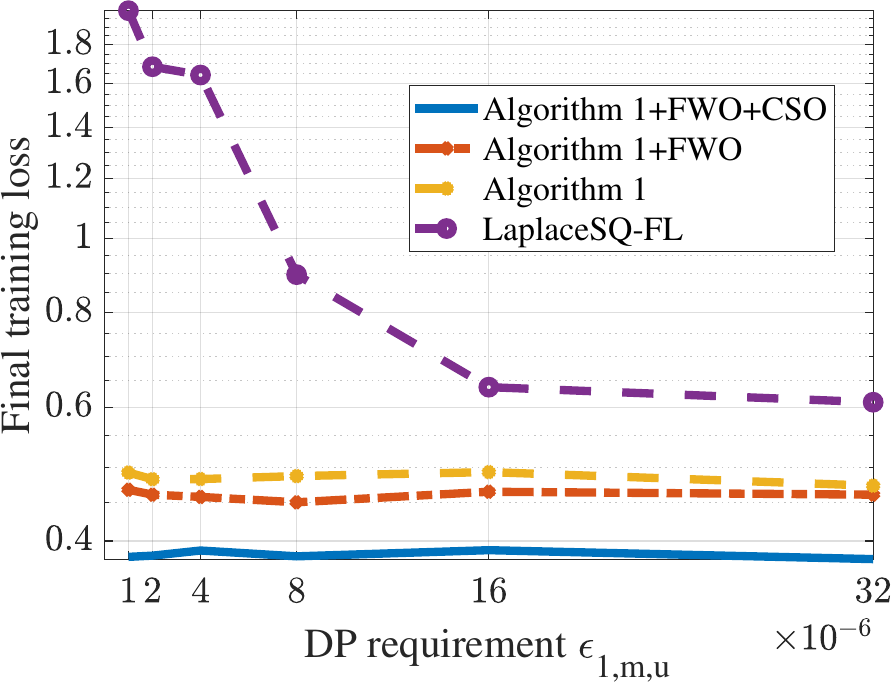}%
        \label{fig:5.3b}%
        }%   
      %  \hfill
    \vspace{-0.2cm}
        \caption{a) Training loss of Algorithm~\ref{algorithm1}+FWO+CSO, Algorithm~\ref{algorithm1}+FWO, Algorithm~\ref{algorithm1}, and LaplaceSQ-FL when $\epsilon_{1,m,u} = 10^{-6},\forall d_{m,u}$. b) and c) Final training loss of Algorithm~\ref{algorithm1}+FWO+CSO, Algorithm~\ref{algorithm1}+FWO, Algorithm~\ref{algorithm1}, and LaplaceSQ-FL for $\epsilon_{1,m,u} \in 10^{-6} \times\{1,2,4,8,16,32\}$ when $\sigma_{2,u'} = 0.125$ and $\sigma_{2,u'} = 1.25\times10^{-2}$, respectively.}
                \label{fig:5.1}

%   \vspace{-8mm}
    \end{figure*}
 \begin{figure*}[htb]
        \centering
        \subfloat[]{%
    \includegraphics[width=0.31\textwidth]{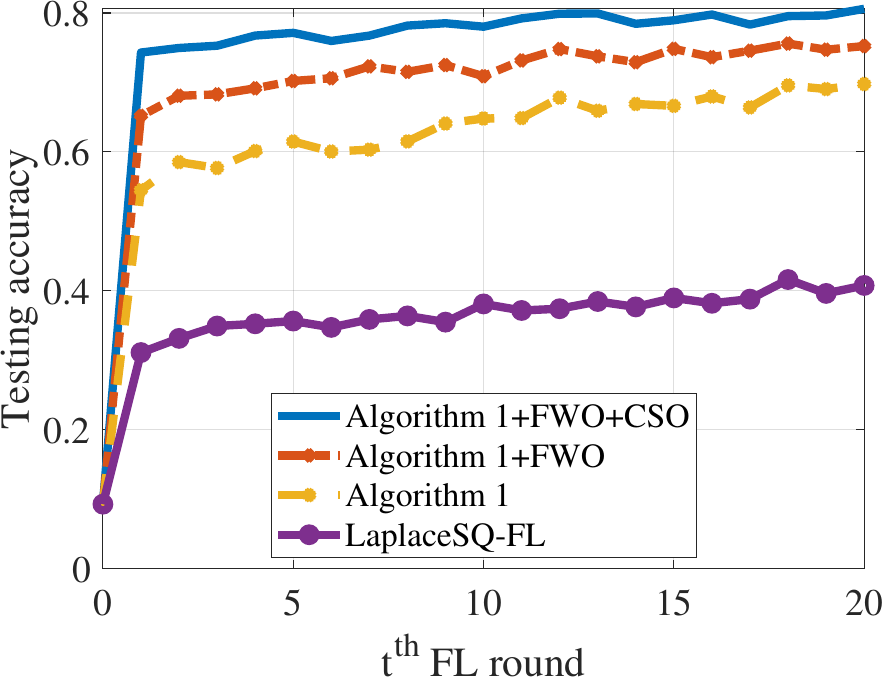}%
     \label{fig:5.1b}
        }%%
    \hspace{2mm}
   \centering
    \subfloat[]{%
    \includegraphics[width=0.31\textwidth]{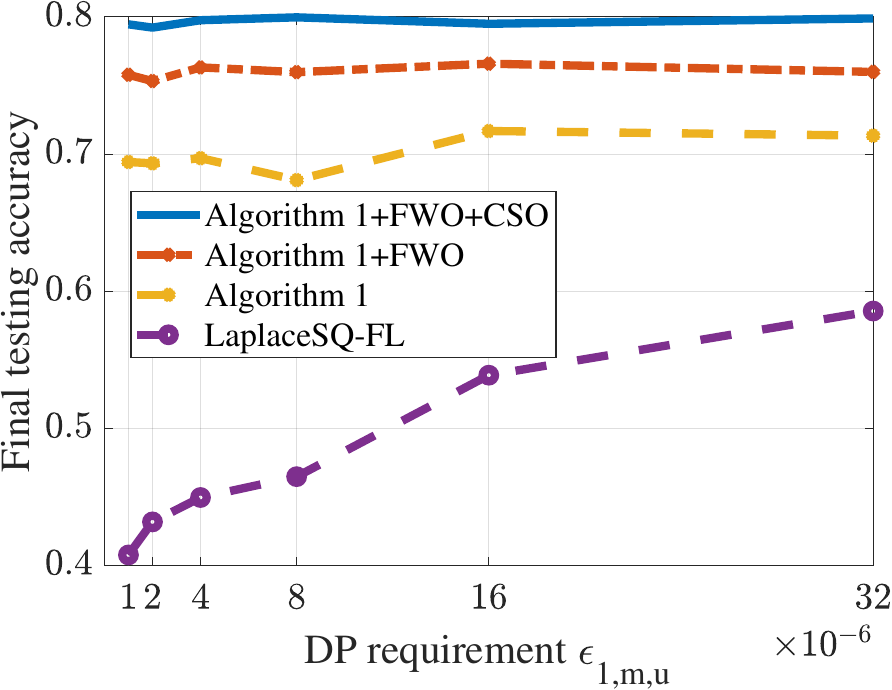}%
     \label{fig:5.2b}
    }%
     \hspace{2mm}
   \centering
     \subfloat[]{%
    \includegraphics[width=0.31\textwidth]{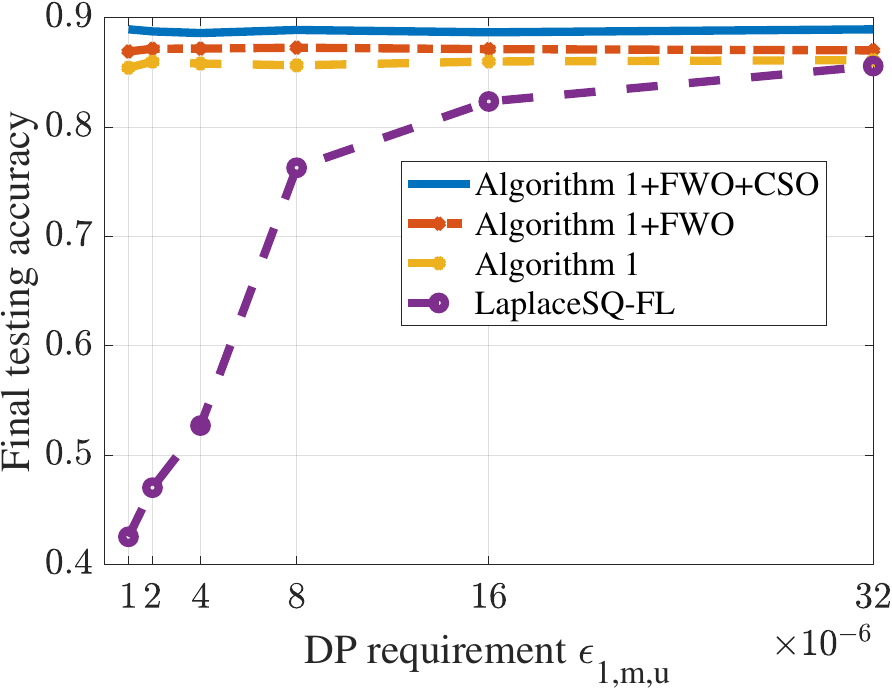}%
     \label{fig:5.3a}
        }%
       \vspace{-0.2cm}
        \caption{a) Testing accuracy of Algorithm~\ref{algorithm1}+FWO+CSO, Algorithm~\ref{algorithm1}+FWO, Algorithm~\ref{algorithm1}, and LaplaceSQ-FL when $\epsilon_{1,m,u} = 10^{-6},\forall d_{m,u}$. b) and c)~Final testing accuracy of Algorithm~\ref{algorithm1}+FWO+CSO, Algorithm~\ref{algorithm1}+FWO, Algorithm~\ref{algorithm1}, and LaplaceSQ-FL for $\epsilon_{1,m,u} \in 10^{-6} \times\{1,2,4,8,16,32\}$ when $\sigma_{2,u'} = 0.125$ and $\sigma_{2,u'} = 1.25\times10^{-2}$, respectively.}
        \label{fig:5.2}
        \vspace{-6mm}
    \end{figure*}
\subsection{Learning Utility Evaluation}
\subsubsection{Numerical Setting}
\textit{ML Model and Data Set:} The experiments are conducted on the MNIST data set \cite{lecun2010} for image classification, using a cross-entropy training loss function to train an ML model.
The ML model consists of two fully connected layers.
The input image is passed through the first fully connected layer with 200 neurons, followed by a ReLU activation function.
The second fully connected layer has 10 neurons, followed by a softmax activation function for classification.
The ML model contains $d = 159,010$ parameters. 
MNIST data set contains $60,000$ training samples and $10,000$ testing data samples.
The training and testing data are i.i.d. across devices with each device owning $D = 600$ training samples and $100$ testing samples.
These devices are divided into two groups based on quantization resolution: 50 devices in $\cG_1$ with {2-bit} quantization and 50 devices in $\cG_2$ with 4-bit quantization, i.e., $M = 2$, $b_1 = 2$, $b_2 = 4$, and $g_1 = g_2 = 50$. 
The total quantization bits constraint, as specified in \eqref{eq:op1a}, is set to $B = 30$ bits. 
The device-to-FC noises are set with standard deviations $\sigma_{1,u} = 6.25\times10^{-4}$ for 2-bit devices ($u \in \cG_1$) and $\sigma_{2,u'} = 0.125$ for 4-bit devices ($u' \in \cG_2$). 
In the following, the aforementioned parameters are fixed unless specified otherwise.\looseness=-1
\textit{Hyperparameters:} The FL process consists of $T = 20$ rounds in Algorithm~\ref{algorithm1}. 
In each round, $N = 10$ devices  are allowed to participate in the learning process as specified in \eqref{eq:op1c}. 
Each device performs $L = 10$ local mini-batch SGD iterations with batch size $|\cB_{l,m,u,t}| = 10, \forall l,d_{m,u},t$ with the $\ell_1$-norm clipping constant $C = 10$. 
\textit{Benchmark}: 
The proposed Algorithm~\ref{algorithm1} is compared with a baseline that combines a standard FL using SQ from \cite{Alistarh2017} and the Laplace mechanism from Lemma~\ref{LpMechanism}. 
This baseline, denoted as LaplaceSQ-FL, serves as a reference to evaluate the effectiveness of the proposed differentially private quantized FL in Algorithm~\ref{algorithm1}.  
In the following simulation, the benchmark LaplaceSQ-FL adopts the sate-of-the-art fusion weight design in \cite{CongShen2021}, which assigns the fusion weights inversely-proportional to quantization resolution.\looseness=-1
We evaluate the testing accuracy and training loss of Algorithm~\ref{algorithm1} under three distinct  configurations. 
The first configuration incorporates both fusion weight optimization (FWO), as described in \eqref{eq:fusionWOpt}, and cluster sizes optimization (CSO), as described in Section~\ref{SecVB}, which is referred to as Algorithm~\ref{algorithm1}+FWO+CSO.  
The second configuration only applies FWO while selecting cluster sizes randomly, which is referred to as  Algorithm~\ref{algorithm1}+FWO.
The third configuration is simply referred to as Algorithm~\ref{algorithm1}, in which the fusion weights are distributed uniformly across model updates, i.e., $\omega_{m,u} = \frac{1}{N}$, and cluster sizes are randomly selected to satisfy \eqref{eq:op1a} and \eqref{eq:op1c}.\looseness=-1  
%%

% %%
%  \begin{figure*}[htb]
%         \centering
%         \subfloat[]{%
%         \includegraphics[width=0.4\textwidth]{Figures/Train_Loss_HetQFL_Eps.pdf}%
%         \label{fig:5.3b}%
%         }%
%     \hspace{2mm}
%       %  \hfill
%         \subfloat[]{%
%         \includegraphics[width=0.4\textwidth]{Figures/Test_Acc_HetQFL_Eps.pdf}%
%      \label{fig:5.3a}
%         }%
%        %%\vspace{-0.5cm}
%         \caption{a) Final training loss and b) Final testing accuracy of Algorithm~\ref{algorithm1}+FWO+CSO, Algorithm~\ref{algorithm1}+FWO, Algorithm~\ref{algorithm1}, and LaplaceSQ-FL for $\epsilon_{1,m,u} \in 10^{-6} \times \{1,2,4,8,16,32\}$ and $\sigma_{2,u'} = 1.25 \times10^{-2} (u' \in \cG_2)$.}
%         \label{fig:5.3}
%     \end{figure*}
% %%
\subsubsection{Training Loss and Testing Accuracy}
Fig.~\ref{fig:5.1a} illustrates the training loss of Algorithm~\ref{algorithm1}+FWO+CSO, Algorithm~\ref{algorithm1}+FWO, Algorithm~\ref{algorithm1}, and the benchmark LaplaceSQ-FL under a  DP requirement of $\epsilon_{1,m,u} = 10^{-6}, \forall d_{m,u}$. 
It can be seen from Fig.~\ref{fig:5.1a} that applying FWO to Algorithm~\ref{algorithm1} improves its training performance, and this improvement is further enhanced with the incorporation of CSO. 
All training loss curves show a steadily decreasing trend as the number of FL rounds $T$ increases. 
The benchmark LaplaceSQ-FL, however, exhibits a higher training loss compared to Algorithm~\ref{algorithm1}+FWO+CSO, Algorithm~\ref{algorithm1}+FWO, and Algorithm~\ref{algorithm1}.
This is attributed to two factors: (i) the high variance of the artificial Laplace noise, required for the DP constraint $\epsilon_{1,m,u} = 10^{-6}$, and (ii)~larger weights assigned to noisy updates from $4$-bit devices according to the fusion approach in \cite{CongShen2021}. %\tk{Regarding (ii), see my previous comment on this issue with the artificial setting. This may be Okay with the initial submission, but be aware of potential risks as the reviewers would point out.}
Fig.~\ref{fig:5.1b} presents the corresponding testing accuracy of Algorithm~\ref{algorithm1}+FWO+CSO, Algorithm~\ref{algorithm1}+FWO, Algorithm~\ref{algorithm1}, and the benchmark LaplaceSQ-FL for the training loss in Fig.~\ref{fig:5.1a}. 
The results in Fig.~\ref{fig:5.1b} show that the proposed approaches exhibit an increase in the testing accuracy as $t$ increases.
Specifically, Algorithm~\ref{algorithm1} itself achieves approximately $70\%$ testing accuracy after 20 FL rounds, which improves to $75\%$ with FWO and further to $80\%$ when both FWO and CSO are applied. 
In contrast, the benchmark LaplaceSQ-FL attains a significantly lower testing accuracy of $41\%$ at the end of the FL process. 
The accuracy trend in Fig.~\ref{fig:5.1b} aligns with the training loss behavior observed in Fig.~\ref{fig:5.1a}.
Fig.~\ref{fig:5.2a} demonstrates the final training loss, obtained after $T=20$ FL rounds, of Algorithm~\ref{algorithm1}+FWO+CSO, Algorithm~\ref{algorithm1}+FWO, Algorithm~\ref{algorithm1}, and LaplaceSQ-FL with $\epsilon_{1,m,u}$ increasing from $10^{-6}$ to $32\times 10^{-6}$.
The final training loss of Algorithm~\ref{algorithm1} is slightly reduced as $\epsilon_{1,m,u}$ increases while those of Algorithm~\ref{algorithm1}+FWO+CSO and  Algorithm~\ref{algorithm1}+FWO remain stable across varying $\epsilon_{1,m,u}$ values.
The benchmark LaplaceSQ-FL exhibits performance improvement as the DP requirement is relaxed (i.e., as $\epsilon_{1,m,u}$ increases).
As the DP requirement becomes stricter (i.e., as $\epsilon_{1,m,u}$ decreases),  LaplaceSQ-FL exhibits a larger final training loss compared to the proposed algorithms.
%%
%Overall, the trend in Fig.~\ref{fig:5.2a} is due to the decrease in artificial Laplace noise variance as $\epsilon_{1,m,u}$ increases and the stable distortion of the proposed SQ as presented in Remark~\ref{rmk1} \tk{This sentence is unclear. I think it would be better to remove}.
%%
%%
Fig.~\ref{fig:5.2b} presents the corresponding final testing accuracy.
The proposed methods outperform LaplaceSQ-FL in testing accuracy across all $\epsilon_{1,m,u}$ values, while LaplaceSQ-FL's testing accuracy improves only when $\epsilon_{1,m,u}$ increases.
Overall, the trends in Figs.~\ref{fig:5.2a}-\ref{fig:5.2b} confirm the benefits of the proposed algorithms, which maintains stable learning performance across different $\epsilon_{1,m,u}$ values.\looseness=-1
\begin{figure*}
\begin{subequations}\small
\label{T2_eq1d0}
    \beq
    \d4\d4 \E \left[ \left\| \whw^{t+1} - \overline{\bw}^{t+1} \right\|^2_2 \right] \d4& = &\d4 \E \left[ \left\| \bw^t + \sumM \sum_{d_{m,u} \in \cC_m}\omega_{m,u}\widehat{\bw}^{m,u,t} - \sumM \sum_{d_{m,u} \in \cC_m} \omega_{m,u}\bw^{L,m,u,t} \right\|^2_2 \right], \nonumber \\
    \d4& = &\d4  \E \left[ \left\| \bw^t + \sumM \sum_{d_{m,u} \in \cC_m}\omega_{m,u}\cQ_{b_m, \epsilon_{1,m,u}}(\widehat{\bv}^{m,u,t}) - \sumM \sum_{d_{m,u} \in \cC_m} \omega_{m,u}\bw^{L,m,u,t} \right\|^2_2 \right], \nonumber\\
    \d4& = &\d4  \E \left[ \left\|\sumM \sum_{d_{m,u} \in \cC_m}\omega_{m,u}\cQ_{b_m, \epsilon_{1,m,u}}(\widehat{\bv}^{m,u,t}) - \sumM \sum_{d_{m,u} \in \cC_m} \omega_{m,u}(\bw^{L,m,u,t}-\bw^t) \right\|^2_2 \right], \nonumber\\
    \d4& \leq & \d4 \scalemath{0.93}{\E_{\{\cC_m\}_{m=1}^M} \left[ \E\left[\Big( \sumM \sum_{d_{m,u} \in \cC_m}\omega^2_{m,u} \Big) \sumM \sum_{d_{m,u} \in \cC_m} \d4 \left\|\cQ_{b_m, \epsilon_{1,m,u}}(\widehat{\bv}^{m,u,t}) - \bv^{m,u,t} \right\|^2_2 \Bigg| \{\cC_m\}_{m=1}^M \right]\right],}\label{T2_eq1a} \\
    \d4& \leq & \d4 \scalemath{0.93}{\E_{\{\cC_m\}_{m=1}^M} \left[ \E\left[ \sumM \sum_{d_{m,u} \in \cC_m} \d4 \left\|\cQ_{b_m, \epsilon_{1,m,u}}(\widehat{\bv}^{m,u,t}) - \bv^{m,u,t} \right\|^2_2 \Bigg| \{\cC_m\}_{m=1}^M \right] \right],}\label{T2_eq1a.1} \\
    \d4& = &\d4 \scalemath{0.93}{\E_{\{\cC_m\}_{m=1}^M} \left[ \sumM \sum_{d_{m,u} \in \cG_m} \d4 \indi_{\cC_m}(d_{m,u})\E\left[ \left\|\cQ_{b_m, \epsilon_{1,m,u}}(\widehat{\bv}^{m,u,t}) - \bv^{m,u,t} \right\|^2_2  \right] \right],}\nonumber \\
    \d4& = &\d4 \scalemath{1}{\sumM \sum_{d_{m,u} \in \cG_m} \frac{c_m}{g_m} \E \left[ \left\|\cQ_{b_m, \epsilon_{1,m,u}}(\widehat{\bv}^{m,u,t}) -\widehat{\bv}^{m,u,t} + \widehat{\bv}^{m,u,t}  - \bv^{m,u,t} \right\|^2_2 \right]},\label{T2_eq1b} \\ 
    \d4& \leq &\d4   \scalemath{1}{2\sumM \sum_{d_{m,u} \in \cG_m} \frac{c_m}{g_m}\E \left[ \left\|\cQ_{b_m, \epsilon_{1,m,u}}(\widehat{\bv}^{m,u,t}) -\widehat{\bv}^{m,u,t}\right\|^2_2\right]}\nonumber \\ &&~~~~~~~~~~~~~~~~~~~~~~~~~~~~~~+ 2 \sumM \sum_{d_{m,u} \in \cG_m} \frac{c_m}{g_m}\E \left[ \left\| \min \left\{1,\frac{C}{\| \bv^{m,u,t}\|_1} \right\}\bv^{m,u,t}  - \bv^{m,u,t} \right\|^2_2 \right], \label{T2_eq1c} \\
    \d4 &\leq& \d4 2 \sumM \sum_{ d_{m,u} \in \cG_m } 
    \frac{c_m}{g_m}
    \frac{4dC^2}{(2^{b_m}-1)^2} 
    + 2\Lambda  \sumM \sum_{d_{m,u} \in \cG_m} \frac{c_m}{g_m} \E \left[ \|\bv^{m,u,t}\|^2_2 \right],\label{T2_eq1d}
    \eeq 
\end{subequations}
\hrulefill
\vspace{-5mm}
\end{figure*}
Figs.~\ref{fig:5.3b}-\ref{fig:5.3a} demonstrates the final training loss and testing accuracy performance of Algorithm~\ref{algorithm1}+FWO+CSO, Algorithm~\ref{algorithm1}+FWO, Algorithm~\ref{algorithm1}, and LaplaceSQ-FL for $\sigma_{2,u'} = 1.25 \times10^{-2} (u' \in \cG_2)$ (reducing link noise of 4-bit devices) and $\epsilon_{1,m,u}$ varying from $10^{-6}$ to $32 \times 10^{-6}$.
Similar trends as in Figs.~\ref{fig:5.2a}-\ref{fig:5.2b} can be observed from Figs.~\ref{fig:5.3b}-\ref{fig:5.3a} except for that the final training loss and testing accuracy performances of LaplaceSQ-FL are relatively improved, compared to Figs.~\ref{fig:5.2a}-\ref{fig:5.2b}, when the DP requirement is relaxed ($\epsilon_{1,m,u} = 32 \times 10^{-6}$). 
This improvement comes from a reduction in the link noise variance of 4-bit devices from $0.125$ (Figs.\ref{fig:5.2a}-\ref{fig:5.2b}) to $1.25\times10^{-2}$ (Figs.~\ref{fig:5.3b}-\ref{fig:5.3a}) and a decrease in the Laplace noise variance when DP requirement is relaxed to $\epsilon_{1,m,u} = 32 \times 10^{-6}$.
Our proposed methods are noteworthy for maintaining a stable learning utility (around $90\%$ accuracy) across both strict and relaxed DP requirements. Importantly, they effectively resolve the inherent learning utility and privacy tradeoffs, especially under stringent privacy constraints (as $\ep_{1,m,u}$ decreases), which is in contrast to LaplaceSQ-FL.
%%
%%

%%
%These simulation results demonstrate the advantages of the proposed framework in terms of distortion reduction, privacy preservation, and model utility.
%%
\section{Conclusion}\label{SecVII}
In this paper, we proposed a unified framework to enhance the learning utility of the privacy-preserving quantized FL algorithm under quantization heterogeneity. 
In our FL network model, clusters of devices with different quantization resolutions are allowed to participate in each FL round. 
We proposed a novel stochastic quantizer that ensures a DP requirement while minimizing the quantization distortion. 
The proposed quantizer leverages the inherent randomness to provide training data privacy protection while achieving minimal distortion. 
Notably, our quantizer exhibits a bounded distortion, contrasting with a conventional DP mechanism (LaplaceSQ) where the distortion can grow unbounded.
{It is because the distortion of the proposed SQ decreases monotonically with DP requirement while the distortion of LaplaceSQ increases.}
To deal with the quantization heterogeneity, we developed a cluster size optimization method aimed at minimizing learning error in combination with a linear fusion technique that improves model aggregation accuracy.
Throughout the paper, numerical simulations demonstrated the effectiveness of the proposed framework, showcasing privacy protection capabilities and improved learning utility compared to existing methods. 
These results highlight the potential of our framework for privacy-preserving quantized FL systems.\looseness=-1
\begin{figure*}
\small
\begin{subequations}\small\label{T2_eq5}
    \beq
 \d4\d4\E[\|\overline{\be}^{m,u,t}\|^2_2]  \d4& = &\d4 \E\left[ \left\|\Wnabla f_{m,u}(\bw^{t}) 
 - \nabla f(\bw^{t})\right\|^2_2\right] + \E\left[\left\|\sum_{l=1}^{L-1} (\Wnabla f_{m,u}(\bw^{l,m,u,t}) - \nabla f(\bw^{l,m,u,t}) + \nabla f(\bw^{l,m,u,t}) - \nabla f(\bb^{l,t}))\right\|^2_2\right], \label{T2_eq5a0} \\ 
 \d4 &\leq& \d4 \zeta + \E\left[\left\|\sum_{l=1}^{L-1} (\Wnabla f_{m,u}(\bw^{l,m,u,t}) \!-\! \nabla f(\bw^{l,m,u,t})) \!+\! \sum_{l=1}^{L-1} (\nabla f(\bw^{l,m,u,t}) \!-\! \nabla f(\bb^{l,t}))\right\|^2_2\right], \label{T2_eq5a}\\ 
  \d4 &=& \d4 \zeta  + \E\left[ \left\| \sum_{l=1}^{L-1}\Wnabla f_{m,u}(\bw^{l,m,u,t}) - \nabla f(\bw^{l,m,u,t})\right\|^2_2\right] + \E\left[\left\|\sum_{l=1}^{L-1} \nabla f(\bw^{l,m,u,t}) - \nabla f(\bb^{l,t})\right\|^2_2\right],\label{T2_eq5a.1}\\
 \d4 &\leq& \d4 \zeta + \sum_{l=1}^{L-1} \E \left[ \left \| \Wnabla f_{m,u}(\bw^{l,m,u,t}) - \nabla f(\bw^{l,m,u,t}) \right\|^2_2 \right] 
 + (L-1)\sum_{l=1}^{L-1} \E \left[ \left \| \nabla f(\bw^{l,m,u,t}) - \nabla f(\bb^{l,t}) \right \|^2_2 \right],\label{T2_eq5b}\\
 &\leq& L\zeta + (L-1)\gamma^2\sum_{l=1}^{L-1}\E\Big[\|\bw^{l,m,u,t} - \bb^{l,t}\|^2_2\Big],\label{T2_eq5c} 
    \eeq 
\end{subequations}
\hrulefill
\vspace{-5mm}
\end{figure*}
\appendices
\section{Proof of Lemma~\ref{lm1}}\label{ProofLM1}
Without loss of generality, we consider the case when $|q_i-a| \leq |q_{i+1}-a|$.
From the constraints in \eqref{eq3c}, we have $\frac{1}{e^{\epsilon_1}+1} \leq p \leq \frac{e^{\epsilon_1}}{e^{\epsilon_1}+1}$. 
It is not difficult to verify that  $p(q_i-a)^2 + (1-p)(q_{i+1}-a)^2 = p[(q_i-a)^2  - (q_{i+1}-a)^2]  +  (q_{i+1}-a)^2 \geq \frac{e^{\epsilon_1}}{e^{\epsilon_1}+1}[(q_i-a)^2  - (q_{i+1}-a)^2]  +  (q_{i+1}-a)^2 = \frac{e^{\epsilon_1}}{e^{\epsilon_1}+1}(q_i-a)^2 + \frac{1}{e^{\epsilon_1}+1}(q_{i+1}-a)^2$, where the last inequality follows from the assumption $|q_i-a|  \leq |q_{i+1}-a|$. 
Thus, the minimum value of the objective function in \eqref{eq3a} is achieved if and only if $p = \frac{e^{\epsilon_1}}{e^{\epsilon_1}+1}$.
The derivation remains analogous for the case when $|q_i-a| \geq |q_{i+1}-a|$. 
This completes the proof.
\section{Proof of Lemma~\ref{lm:2ndTerm}}\label{AppTerm2}
The term $\E[\| \whw^{t+1} - \overline{\bw}^{t+1} \|^2_2]$ can be bounded as in \eqref{T2_eq1d0}, where \eqref{T2_eq1a} follows from \eqref{eqCS} and Step 10 of Algorithm~\ref{algorithm1};
\eqref{T2_eq1a.1} follows from the fact that $\sumM \sum_{d_{m,u} \in \cC_m}\omega_{m,u} = 1$;  
\eqref{T2_eq1b} follows from $\Pr[d_{m,u} \in \cC_{m}] = \frac{1}{g_m}, \forall d_{m,u} \in \cG_m$;
\eqref{T2_eq1c} follows from the arithmetic-geometric mean (AGM) inequality and Step~\ref{step11} of Algorithm~\ref{algorithm1}; %\tk{I think the CS inequality is predominantly applied in every proof which doesn't seem to be interesting. This part of bound we would refer to the arithmetic-geometric mean (AGM) inequality instead of CS inequality when there are two variables, which is essentially the same as CS inequality.} 
and \eqref{T2_eq1d} follows from Assumption~\ref{assumption4} and from the fact that, for $\bx \in \R^{d \times 1}$ with $\|\bx\|_{1} \leq C$, $\E[\|\cQ_{b,\epsilon_1}(\bx) - \bx\|_2^2] = \sum_{i=1}^{d}\E[|\cQ_{b,\epsilon_1}(x_i)-x_i|^2] \leq \frac{4dC^2}{(2^{b}-1)^2}$. 
The last bound is due to the distortion in \eqref{eq:g_eps1} and the fact that $\max\{(q_i-a)^2,(q_{i+1}-a)^2\} \leq \frac{(\overline{a} - \underline{a})^2}{(2^b-1)^2}$.
From Step \ref{step10} of Algorithm~\ref{algorithm1}, the term $\|\bv^{m,u,t}\|^2_2$ in \eqref{T2_eq1d} can be rewritten as $\|\bv^{m,u,t}\|^2_2 = \|\bw^{L,m,u,t}-\bw^t\|^2_2 =  \big\| \bw^t - \eta_t\sum_{l=0}^{L-1} \Wnabla f_{m,u}(\bw^{l,m,u,t}) - \bw^t\big\|^2_2$ $=  \big\| \eta_t\sum_{l=0}^{L-1} \Wnabla f_{m,u}(\bw^{l,m,u,t})\big\|^2_2$.
Defining $\overline{\be}^{m,u,t} = \sum_{l=0}^{L-1} \big(\Wnabla f_{m,u}(\bw^{l,m,u,t}) -\nabla f(\bb^{l,t}) \big)$ yields
\begin{subequations}\small
    \beq
    \left \|\bw^{L,m,u,t}-\bw^t \right\|^2_2 \d4& = &\d4 \left\| \eta_t(\bg^t + \overline{\be}^{m,u,t})\right\|^2_2, \label{T2_eq2a} \\ 
    \d4& \leq &\d4 2\eta_t^2 \|\bg^t\|^2_2 + 2\eta_t^2\|\overline{\be}^{m,u,t}\|^2_2, \label{T2_eq2}
    \eeq
\end{subequations} where \eqref{T2_eq2a} is due to {\small $\bg^t = \sum_{l=0}^{L-1} \nabla f(\bb^{l,t})$} and \eqref{T2_eq2} is due to the AGM inequality. %\tk{change this to the AGM ineuqlity}.
Plugging \eqref{T2_eq2} into \eqref{T2_eq1d} and noting that {\small $\sumM \sum_{d_{m,u} \in \cG_m}\frac{c_m}{g_m} = N$} and $\frac{c_m}{g_m} \leq 1$ gives \looseness=-1
\begin{multline}
\d4\!\!\scalemath{0.9}{\E \Big[ \| \whw^{t+1} - \overline{\bw}^{t+1} \|^2_2 \Big]  \leq  8 d C^2 \sumM \frac{c_m}{(2^{b_m}-1)^2} } \\ \scalemath{0.9}{+ 4 \Lambda N \eta_t^2 \E[\|\bg^t\|^2_2]  \!+\! 4 \Lambda \eta_t^2 \sumM \sum_{d_{m,u} \in \cG_m}\d4\!\! \E[\|\overline{\be}^{m,u,t}\|^2_2].}\d4   \label{T2_eq3}
\end{multline}
Next, we aim to find upper bounds for $\E[\|\bg^t\|^2_2]$ and $\E[\|\overline{\be}^{m,u,t}\|^2_2]$ in \eqref{T2_eq3}. 
Applying Jensen's inequality to $\ell_2$-norm yields {\small $\E[\|\bg^t\|^2_2] = \E[\|\sum_{l=0}^{L-1}\nabla f(\bb^{l,t})\|^2_2] \leq L \sum_{l=0}^{L-1}\E[\|\nabla f(\bb^{l,t})\|^2_2] = L \sum_{l=0}^{L-1}\E[\|\nabla f(\bb^{l,t}) - \nabla f(\bw^{\star})\|^2_2]$}.
Therefore, the following holds 
\begin{subequations}\small
\beq
    \E[\|\bg^t\|^2_2] & \leq & L \sum_{l=0}^{L-1} \gamma^2(1-\mu\eta_t)^{l} \E[\|\bw^t - \bw^{\star}\|^2_2], \label{T2_eq4.1} \\ 
    & \leq & L^2\gamma^2 \E[\|\bw^t - \bw^{\star}\|^2_2],\label{T2_eq4} 
\eeq
\end{subequations}where \eqref{T2_eq4.1} is due to Assumption~\ref{assumption1} and \eqref{T1_eq6}; and \eqref{T2_eq4} follows from $0 \leq  (1-\mu\eta_t)^{l} \leq 1$.
From $\bb^{0,t} = \bw^t$, we rewrite {\small $\E[\|\overline{\be}^{m,u,t}\|^2_2] = \E\big[\|\Wnabla f_{m,u}(\bw^{t}) 
 - \nabla f(\bw^{t}) + \sum_{l=1}^{L-1} \big(\Wnabla f_{m,u}(\bw^{l,m,u,t}) - \nabla f(\bw^{l,m,u,t}) + \nabla f(\bw^{l,m,u,t}) - \nabla f(\bb^{l,t})\big)\|^2_2\big]$}. 
 Thus, we have the bound in \eqref{T2_eq5}, where \eqref{T2_eq5a0}, \eqref{T2_eq5a}, and \eqref{T2_eq5a.1} follow from  Assumption~\ref{assumption3}; 
\eqref{T2_eq5b} follows from Assumption~\ref{assumption3} and Jensen's inequality;  
and \eqref{T2_eq5c} follows from Assumption~\ref{assumption1} and Assumption~\ref{assumption3}.
From \eqref{T2_eq5c}, we have {\small $\sumM \sum_{d_{m,u} \in \cG_m} \E[\|\overline{\be}^{m,u,t}\|^2_2] \leq  KL\zeta + (L-1)\gamma^2 \sum_{l=1}^{L-1}\sumM \sum_{d_{m,u} \in \cG_m} \E\Big[\|\bw^{l,m,u,t} - \bb^{l,t}\|^2_2\Big]$}. 
Thus, from the definition of $a^{l,t}$ in \eqref{eq:defat} and the bound in \eqref{T1_eq10}, the following holds
\begin{subequations}\small
    \beq
    \!\!\sumM \!\sum_{d_{m,u} \in \cG_m} \!\d4 \!\!\E[\|\overline{\be}^{m,u,t}\|^2_2] 
    \!\!\d4&\leq&\!\!\d4 KL\zeta \!+\! K(L-1)\gamma^2 \sum_{l=1}^{L-1}a^{l,t},  \nonumber \\  
    \!\!\!\d4& \leq &\!\!\d4 KL\zeta \!+\! KL^2\gamma^2\eta_t^2L\zeta(1+L\eta_t^2\gamma^2)^L, \nonumber\\
    \d4& \leq &\!\!\d4 KL\zeta \!+\! KL^3\gamma^2\eta_t^2\zeta e,  \label{T2:eq6b}
    \eeq 
\end{subequations}
where \eqref{T2:eq6b} follows from the fact $(1+x) \leq e^x$ for $x \geq 0$, and $\eta_t \leq \frac{1}{L\gamma}$. 
%%%
Plugging \eqref{T2:eq6b} and \eqref{T2_eq4} into \eqref{T2_eq3} gives {\small $\E \Big[ \| \whw^{t+1} - \overline{\bw}^{t+1} \|^2_2 \Big] \leq 8 d C^2 \sumM \frac{c_m}{(2^{b_m}-1)^2} + 4\Lambda N \eta_t^2L^2\gamma^2 \E[\|\bw^t - \bw^{\star}\|^2_2] +  4\Lambda K \eta_t^2(L\zeta + L^3\gamma^2\zeta\eta_t^2e)$}, which completes the proof.\sloppy
\section{Proof of Lemma~\ref{lm:3rdTerm}}\label{AppTerm3}
From the definitions of $\bw^{t+1}$ and $\whw^{t+1}$, the following holds
\begin{multline}
\label{T3_eq1} \scalemath{0.9}{ 
    \E[\|\bw^{t+1} - \whw^{t+1}\|_2^2]  =  \E\left[ \left\|\sumM \sum_{d_{m,u} \in \cC_m} \omega_{m,u}\bn^{m,u,t} \right\|^2_2\right]}, \\
    ~~~~~~~~~~~~~~ \scalemath{0.9}{= \E\!\left[ \left\|\sumM \sum_{d_{m,u} \in \cG_m}\!\!\! \indi_{\cC_m}(d_{m,u})  \omega_{m,u}\bn^{m,u,t} \right\|^2_2\right].}\!\!\!\!\!\!
\end{multline}
%%%
Using conditional expectation, \eqref{T3_eq1} can be written as $ {\small \E\Big[ \Big\|\sumM \sum_{d_{m,u} \in \cG_m} \indi_{\cC_m}(d_{m,u}) \omega_{m,u}\bn^{m,u,t} \Big\|^2_2\Big] =} $ $\scalemath{0.85}{\E_{\{\cC_m\}_{m=1}^M} \! \Big[ \E_{\{\bn^{m,u,t}\}}\!\Big[ \Big\|\sumM \! \sum_{d_{m,u} \in \cG_m}\!\! \indi_{\cC_m}(d_{m,u}) \omega_{m,u}\bn^{m,u,t} \Big\|^2_2\Big| \{\cC_m\}_{m=1}^M \Big] \Big]}$. 
%\tk{This equation exceeds the column margin}. 
%%
Therefore, the following holds \sloppy
%%%
\begin{subequations}\small
    \beq
    \d4\E && \d4\d4\d4 \left[ \left\|\sumM \sum_{d_{m,u} \in \cG_m} \indi_{\cC_m}(d_{m,u}) \omega_{m,u}\bn^{m,u,t} \right\|^2_2\right],\nonumber  \\ 
     && ~~~=\E_{\{\cC_m\}_{m=1}^M} \left[ \sumM \sum_{d_{m,u} \in \cG_m}\!\!\! \indi_{\cC_m}(d_{m,u}) \omega^2_{m,u}d\sigma_{m,u}^2 \right]\!\! ,~~~~~\label{T3_eq5a}\\
     && ~~~\leq d\sumM \sum_{d_{m,u} \in \cG_m} \frac{c_m}{g_m}\sigma_{m,u}^2,~~ \d4\d4\d4\label{T3_eq5}
    \eeq 
\end{subequations}where \eqref{T3_eq5a} follows from the fact that $\bn^{m,u,t} \sim \cN(\mathbf{0}, \sigma_{m,u}^2\bI)$ and $\{\bn^{m,u,t}\}$ are independent; 
and \eqref{T3_eq5} follows from the fact that $\Pr[d_{m,u} \in \cC_m] = \frac{1}{g_m}$ and $0 \leq \omega_{m,u} \leq 1, \forall d_{m,u}$. 
Substituting \eqref{T3_eq5} into \eqref{T3_eq1} gives $\E[\|\bw^{t+1} - \whw^{t+1}\|_2^2] \leq d\sumM \sum_{d_{m,u} \in \cG_m} \frac{c_m}{g_m}\sigma_{m,u}^2$, which completes the proof.
\section{Proof of Theorem~\ref{theorem1}}\label{AppTheorem1}
For $m_{1,t}$ and $t\geq t_0$, we have the following upper bound,
\begin{subequations}
    \beq
 m_{1,t} & = & (1+K\eta_t^2)(1-\mu \eta_t)^L + 4\eta_t^2 \Lambda N L^2\gamma^2, \nonumber\\
   & \leq & (1+K\eta_t^2)e^{-L\mu\eta_t} +  4\eta_t^2 \Lambda NL^2\gamma^2, \nonumber \\
    & \leq & (1+K\eta_t^2)(1 -L\mu\eta_t + L^2\mu^2\eta^2_t) +  4\eta_t^2 \Lambda N L^2\gamma^2, \nonumber \\
   & = & 1 - L \mu  \eta_t + L^2 \mu^2  \eta^2_t, \nonumber \\ & & + K\eta_t^2(1 -  L\mu \eta_t +  L^2\mu^2 \eta^2_t)+ 4\Lambda N L^2\gamma^2\eta_t^2 , \nonumber \\ 
   & \leq & 1- L\mu \eta_t + L^2\mu^2\eta_t^2   + K\eta_t^2+  4\Lambda N L^2\gamma^2\eta_t^2, \label{Theorem1_eq1a} \\
   & = & 1- L\mu \eta_t + (L^2\mu^2 + K+4\Lambda N L^2\gamma^2)\eta_t^2, \nonumber\\
   & \leq & 1-\frac{ L\mu \eta_t}{2}\label{Theorem1_eq1b}, \\ 
   & = & 1 - \frac{4}{t},\label{Theorem1_eq1}
    \eeq 
\end{subequations} where \eqref{Theorem1_eq1a} follows from  $\eta_t \leq \eta_{t_0}  \leq \frac{1}{L\mu}$, \eqref{Theorem1_eq1b} follows from $\eta_t \leq \frac{L\mu}{2(L^2\mu^2 + K + 4\Lambda N L^2\gamma^2)}$, and \eqref{Theorem1_eq1} is due to $\eta_t = \frac{8}{ L\mu t}$.  
From \eqref{eq:ub5} and \eqref{Theorem1_eq1}, we have 
\begin{equation}\small
\label{Theorem1_eq2}
\E[\|\bw^{t+1}-\bw^{\star}\|^2_2]  \leq  \left( 1 - \frac{4}{t}\right)\E[\|\bw^t - \bw^{\star}\|^2_2] + m_{2,t}  + \Delta.
\end{equation}
Next, we proceed by induction to show that \eqref{eqtheorem1} holds. 
For $t = t_0$, the inequality in \eqref{eqtheorem1} holds trivially.
Assume that \eqref{eqtheorem1} holds for $t = t_1 > t_0$. 
We verify that it also holds for $t = t_1+1$ in the following.
Applying \eqref{Theorem1_eq2} and the induction hypothesis for $t_1$ gives {\small $\E[\|\bw^{t_1+1}-\bw^{\star}\|^2_2]  \leq    \Big(1 - \frac{4}{t_1}\Big) \E[\|\bw^{t_1} - \bw^{\star}\|^2_2] + m_{2,t_1}  + \Delta \leq  \Big(1-\frac{4}{t_1}\Big)\Big(\frac{t_0}{t_1}\E[\|\bw^{t_0} - \bw^{\star}\|^2_2] + m_{3,{t_1}} + (t_1-4)\Delta \Big) + m_{2,t_1}  + \Delta$}.
It is noted that $m_{2,t_1} = \frac{k_1}{t_1^2} + \frac{k_2}{t_1^4}$.
Thus, the following holds
\begin{equation}\small\label{eq:lastbound}
    \begin{aligned}
     \d4\E[\|&\bw^{t_1+1}-\bw^{\star}\|^2_2] \\
     \d4& \leq  \left(1-\frac{4}{t_1}\right)\frac{t_0}{t_1}\E[\|\bw^{t_0} - \bw^{\star}\|^2_2] + \left(1-\frac{4}{t_1}\right)m_{3,t_1} \\
     \d4&~~~~~~~~~~~~~~~~~~~~~~~~~~~~~~~~~~~~~~~~~~~~~~~+ m_{2,t_1}+ (t_1-3)\Delta,  \\
     \d4& =  \left(1-\frac{4}{t_1}\right)\frac{t_0}{t_1}\E[\|\bw^{t_0} - \bw^{\star}\|^2_2] + \left(1-\frac{4}{t_1}\right)\left(\frac{k_1}{t_1}+\frac{k_2}{t_1^3}\right) \\ 
     \d4& ~~~~~~~~~~~~~~~~~~~~~~~~~~~~~~~~~~~~~~~~~~~+ \frac{k_1}{t_1^2} + \frac{k_2}{t_1^4} + (t_1-3) \Delta, \\
     \d4& = \left(1-\frac{4}{t_1}\right)\frac{t_0}{t_1}\E[\|\bw^{t_0} - \bw^{\star}\|^2_2] + \frac{(t_1-3)k_1}{t_1^2} \\
    \d4& ~~~~~~~~~~~~~~~~~~~~~~~~~~~~~~~~~~~~~~~~+ \frac{(t_1-3)k_2}{t_1^4} + (t_1-3)\Delta, \\
    \d4& \leq  \frac{t_0}{t_1 \!+\! 1}\E[\|\bw^{t_0} \!-\! \bw^{\star}\|^2_2] + m_{3,t_1+1}  \!+\! (t_1\!+\!1\! -\! 4)\Delta,
\end{aligned}
\end{equation}
where the last inequality is due to the fact that $(1-\frac{4}{t_1})\frac{t_0}{t_1} \leq \frac{t_0}{t_1+1}$, $\frac{t_1-3}{t_1^2} \leq \frac{1}{t_1+1}$, and $\frac{t_1-3}{t_1^4} \leq \frac{1}{(t_1+1)^3}$.
The bound in \eqref{eq:lastbound} indicates that \eqref{eqtheorem1} also holds for $(t_1+1)$th FL iteration. 
This completes the induction proof.

\bibliographystyle{IEEEtran} 
\bibliography{biblib}

\end{document}